%% file: DNT_ICLR_2026.tex
\documentclass{article} 
\usepackage{iclr2025_conference,times}

\input{math_commands.tex}

\usepackage{tcolorbox}
\usepackage{colortbl}

\usepackage{thmtools}
\usepackage{thm-restate}

\usepackage{multirow}
\usepackage{animate}
\usepackage{makecell}
\usepackage{thmtools}
\usepackage{thm-restate}
\usepackage{multirow}
\usepackage[utf8]{inputenc} 
\usepackage[T1]{fontenc}    
\usepackage{hyperref}       
\usepackage{url}            
\usepackage{booktabs}       
\usepackage{amsfonts}       
\usepackage{nicefrac}       
\usepackage{microtype}      
\usepackage{graphicx}

\usepackage{enumerate}

\usepackage{tikz}
\usepackage{graphics}
\usepackage{color}
\usepackage{float}
\usepackage{caption}
\usepackage{subcaption}
\usepackage{mathtools, cases}
\usepackage{annotate-equations}
\usepackage{enumitem,kantlipsum}

\usepackage{lipsum}
\usepackage{amsmath}
\usepackage{array, caption, tabularx, makecell, booktabs}%
\usepackage[framemethod=TikZ]{mdframed}
\usepackage{amsthm}
\usepackage{multirow}
\usepackage{wrapfig}
\usepackage{lipsum}  

\usepackage{enumerate}

\usepackage{enumitem}

\usepackage{algorithm}
\usepackage{algpseudocode}

\usepackage{listings}
\usepackage{tikz}
\usetikzlibrary{positioning, arrows.meta}
\usepackage{ifthen}
\usepackage{xspace}

\usepackage[T1]{fontenc}
\usepackage[utf8]{inputenc}
\usepackage{mathtools}
\usepackage{amssymb}

\usepackage{epigraph}

\newcommand{\bW}{\boldsymbol{W}}
\newcommand{\bw}{\boldsymbol{w}}
\newcommand{\bX}{\boldsymbol{X}}
\newcommand{\bY}{\boldsymbol{Y}}

\newcommand{\bU}{\boldsymbol{U}}

\newcommand{\bI}{\boldsymbol{I}}
\newcommand{\bx}{\boldsymbol{x}}
\newcommand{\by}{\boldsymbol{y}}
\newcommand{\bz}{\boldsymbol{z}}

\newcommand{\bh}{\boldsymbol{h}}

\newcommand{\bq}{\boldsymbol{q}}
\newcommand{\bk}{\boldsymbol{k}}

\newcommand{\bi}{\boldsymbol{i}}

\newcommand{\bg}{\boldsymbol{g}}
\newcommand{\bmm}{\boldsymbol{m}}

\newcommand{\bQ}{\boldsymbol{Q}}
\newcommand{\bK}{\boldsymbol{K}}
\newcommand{\bV}{\boldsymbol{V}}
\newcommand{\bP}{\boldsymbol{P}}
\newcommand{\bA}{\boldsymbol{A}}

\newcommand{\bC}{\boldsymbol{C}}
\newcommand{\bJ}{\boldsymbol{J}}

\newcommand{\ie}{\textit{i}.\textit{e}.}
\newcommand{\eg}{\textit{e}.\textit{g}.}

\newlength{\maxwidth}
\newcommand{\algalign}[2]%
{\makebox[\maxwidth][r]{$#1{}$}${}#2$}

\definecolor{blue}{rgb}{0, 0, .7}
\definecolor{darkred}{rgb}{.5, 0, 0}
\definecolor{darkgreen}{rgb}{0, .5, 0}
\definecolor{blue}{rgb}{0, 0, .7}
\definecolor{darkred}{rgb}{.5, 0, 0}
\definecolor{darkgreen}{rgb}{0, .5, 0}
\definecolor{cred}{rgb}{0.8, 0.0, 0.0}
\definecolor{cpink}{rgb}{0.98, 0.38, 0.5}
\definecolor{YellowDot}{RGB}{253, 198, 10}
\definecolor{BlueDot}{RGB}{56, 108, 176}
\definecolor{GreenDot}{RGB}{77, 175, 74}
\definecolor{green2}{rgb}{0.21,0.74,0.49}
\definecolor{commentcolor}{RGB}{110,154,155}   

\DeclareMathAlphabet{\mathcal}{OMS}{cmsy}{m}{n}
\SetMathAlphabet{\mathcal}{bold}{OMS}{cmsy}{b}{n}

\theoremstyle{break}
\newtheorem{theorem}{Theorem}

\newtheorem{proposition}{Proposition}

\newtheorem{remark}{Remark}

\title{\centering DNT: a Deeply Normalized Transformer \\that can be trained by Momentum SGD}

\author{Xianbiao Qi$^{1}$, Marco Chen$^{2}$, Wenjie Xiao$^{3}$, Jiaquan Ye$^{1}$, Yelin He$^{1}$,
\textbf{Chun-Guang Li}$^{4}$,\\
\textbf{Zhouchen Lin$^{5}$}\\
\textsuperscript{1}Intellifusion Inc.,
\textsuperscript{2}Tsinghua University, 
\textsuperscript{3}Johns Hopkins University, \\
\textsuperscript{4}Beijing University of Posts and Telecommunications, 
\textsuperscript{5}Peking University\\
}

\iclrfinalcopy 
\begin{document}

\maketitle
\renewcommand{\thefootnote}{}

\begin{abstract}
{
Transformers have become the de facto backbone of modern deep learning, yet their training typically demands an advanced optimizer with adaptive learning rate like AdamW, rather than a momentum SGDW (mSGDW). 
Previous works show that it is mainly due to a heavy-tailed distribution of the gradients. 
In this paper, we introduce a Deeply Normalized Transformer (DNT), which is meticulously engineered to overcome this limitation enabling seamless training with vanilla mSGDW while yielding comparable performance to the Transformers trained via AdamW. 
To be specific, in DNT, we strategically integrate normalization techniques at proper positions in the Transformers to effectively modulate the Jacobian matrices of each layer, balance the influence of weights, activations, and their interactions, and thus enable the distributions of gradients concentrated. 
We provide both theoretical justifications of the normalization technique used in our DNT and extensive empirical evaluation on two popular Transformer architectures to validate that: a) DNT outperforms its counterparts (\ie, ViT and GPT), and b) DNT can be effectively trained with vanilla mSGDW.
}
\end{abstract}

\section{Introduction}
\label{sec_introduction}
Transformer~\citep{transformer_vaswani2017attention} 
has revolutionized numerous domains in artificial intelligence and 
as the de facto backbone of modern deep learning has demonstrated remarkable capabilities across natural language processing \citep{gpt1_radford2018improving, gpt2_radford2019language, gpt3_brown2020language, llama3_dubey2024llama, qwen_team2023qwen,  deepseek_v3_liu2024deepseek}, computer vision \citep{vit_dosovitskiy2020image, swinv2_liu2022swin,scaling22b_dehghani2023scaling}, AIGC~\citep{dalle1_ramesh2021zero, dit_peebles2023scalable}, and multi-modal applications~\citep{blip_li2022blip, llava_liu2023visual}.

Nowadays it is widely accepted that Adam~\citep{adam_kingma2014adam} or its descendant AdamW~\citep{adamw_IlyaLoshchilov2018FixingWD} are 
the 
standard optimizer for training Transformers; whereas the classical SGD~\citep{sgd_robbins1951stochastic} and its variants~\citep{nesterov1983method, nesterov1998introductory, svrg_johnson2013accelerating}, \eg, momentum SGD (mSGD), usually under-perform when training Transformers.  Indeed, Adam is used as the optimizer in most recent studies on Large Language Models (LLMs)~\citep{llama3_dubey2024llama, qwen_team2023qwen,  deepseek_v3_liu2024deepseek} and multi-modal models~\citep{blip_li2022blip, llava_liu2023visual}, 
despite of its heavier load 
on GPU memory than mSGD.
\begin{figure}[htbp]
	\centering
	\begin{minipage}{0.32\linewidth}
		\centering
		\includegraphics[width=0.98\linewidth,height=0.80\linewidth]{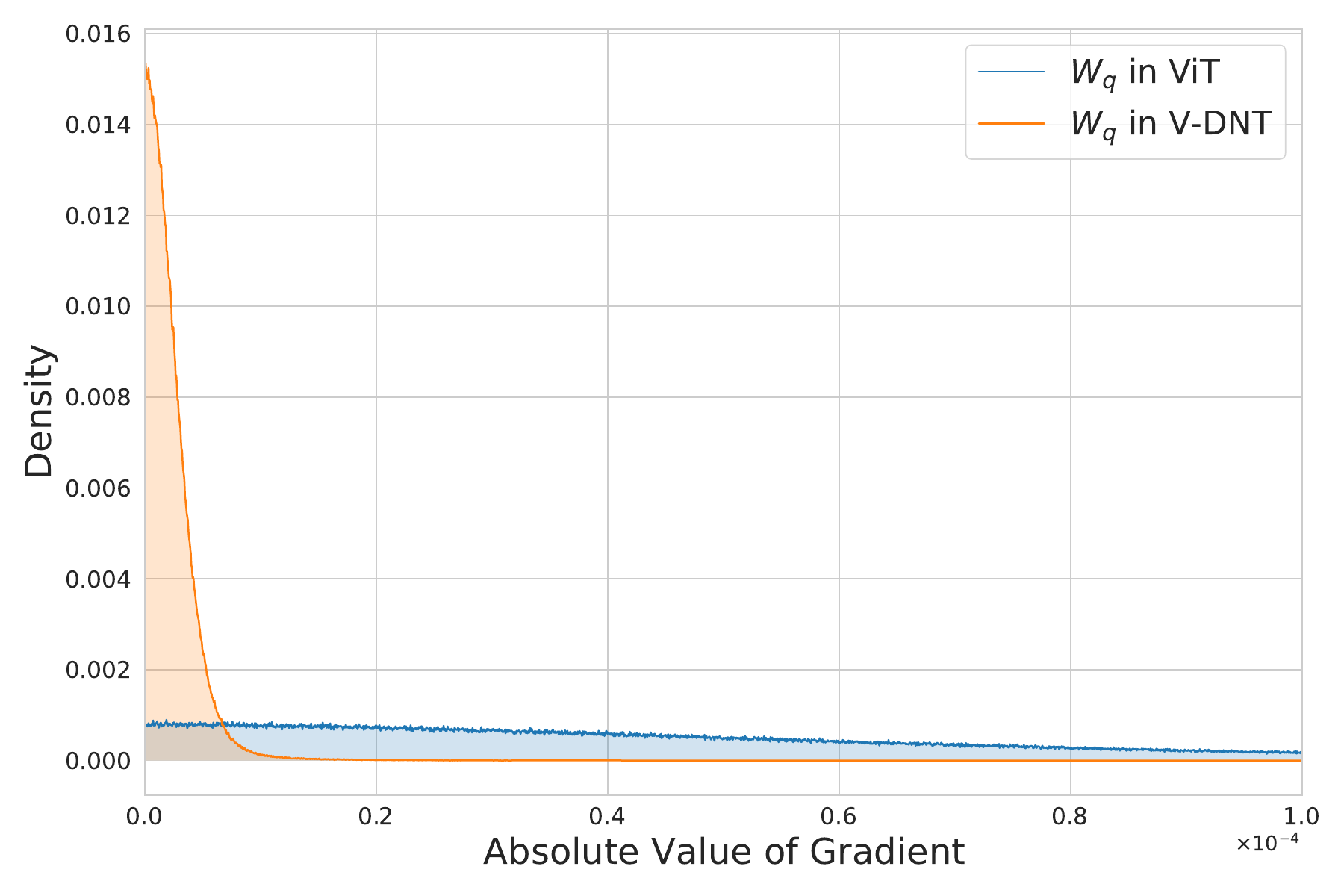}
	\end{minipage}
	\begin{minipage}{0.32\linewidth}
		\centering
		\includegraphics[width=0.98\linewidth,height=0.80\linewidth]{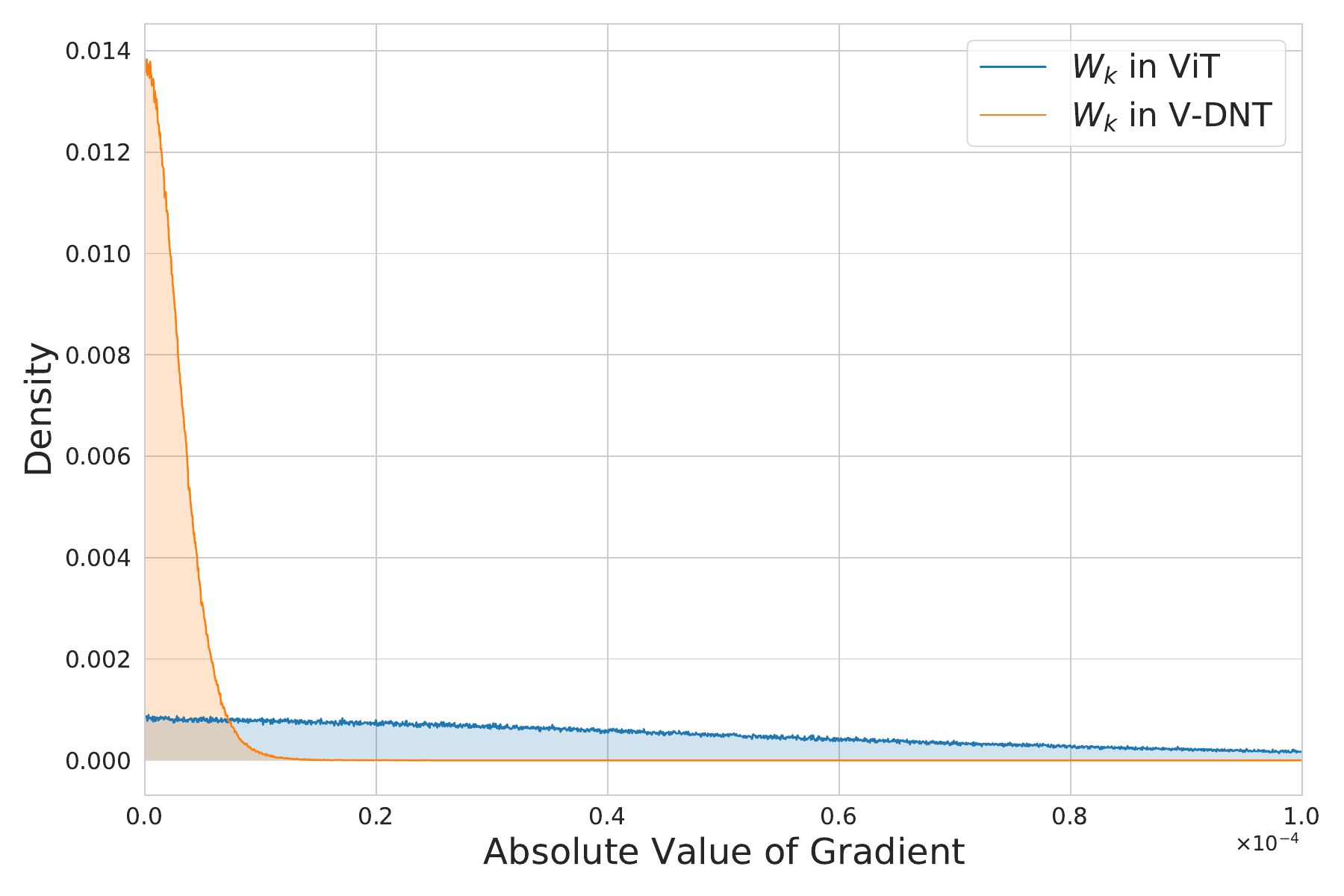}
	\end{minipage}
	\begin{minipage}{0.32\linewidth}
		\centering
		\includegraphics[width=0.98\linewidth,height=0.80\linewidth]{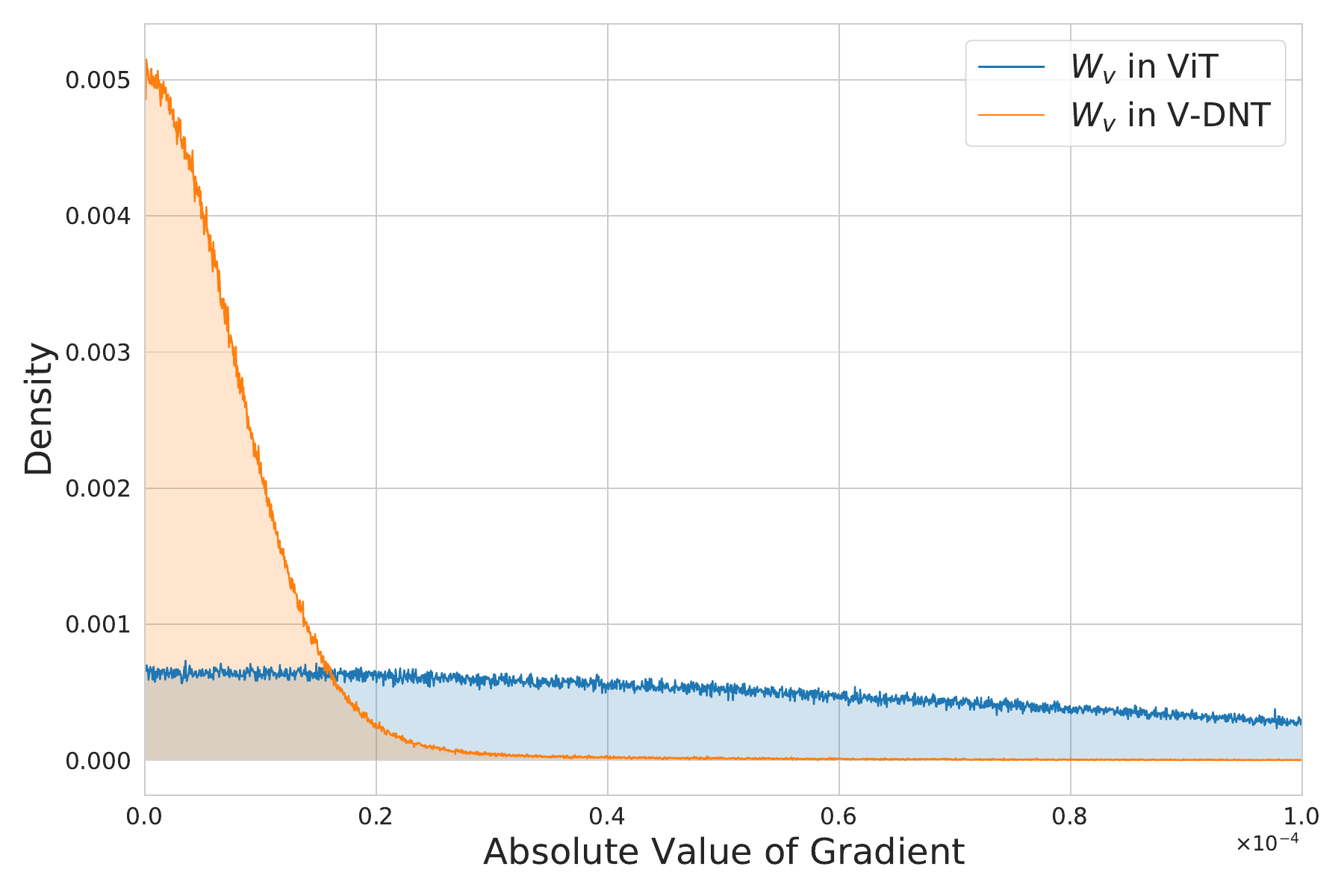}
	\end{minipage}
	\begin{minipage}{0.32\linewidth}
		\centering
		\includegraphics[width=0.98\linewidth,height=0.80\linewidth]{figures/vHeat/visual/wv.pdf}
	\end{minipage}
	\begin{minipage}{0.32\linewidth}
		\centering
		\includegraphics[width=0.98\linewidth,height=0.80\linewidth]{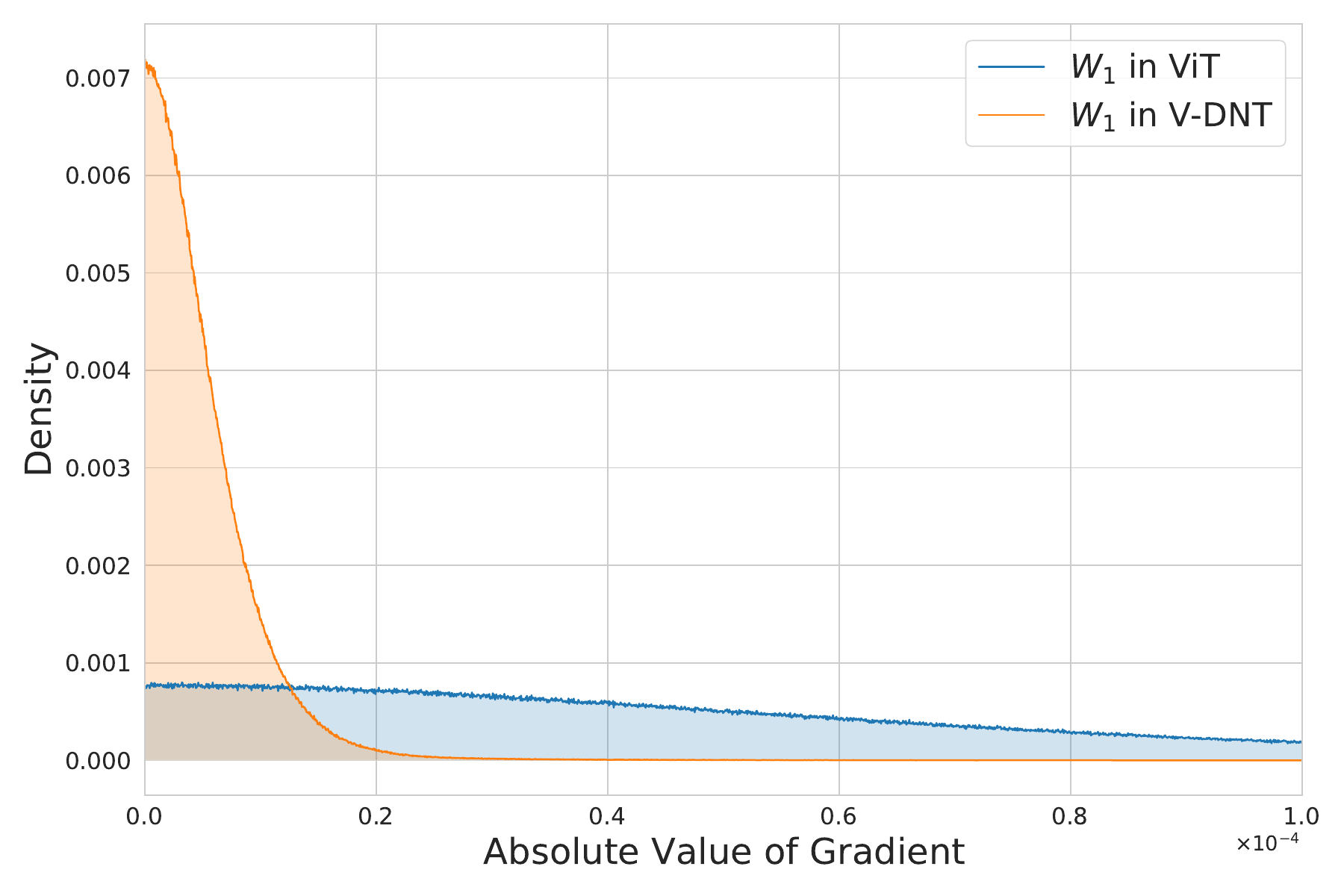}
		\label{chutian2}
	\end{minipage}
	\begin{minipage}{0.32\linewidth}
		\centering
		\includegraphics[width=0.98\linewidth,height=0.80\linewidth]{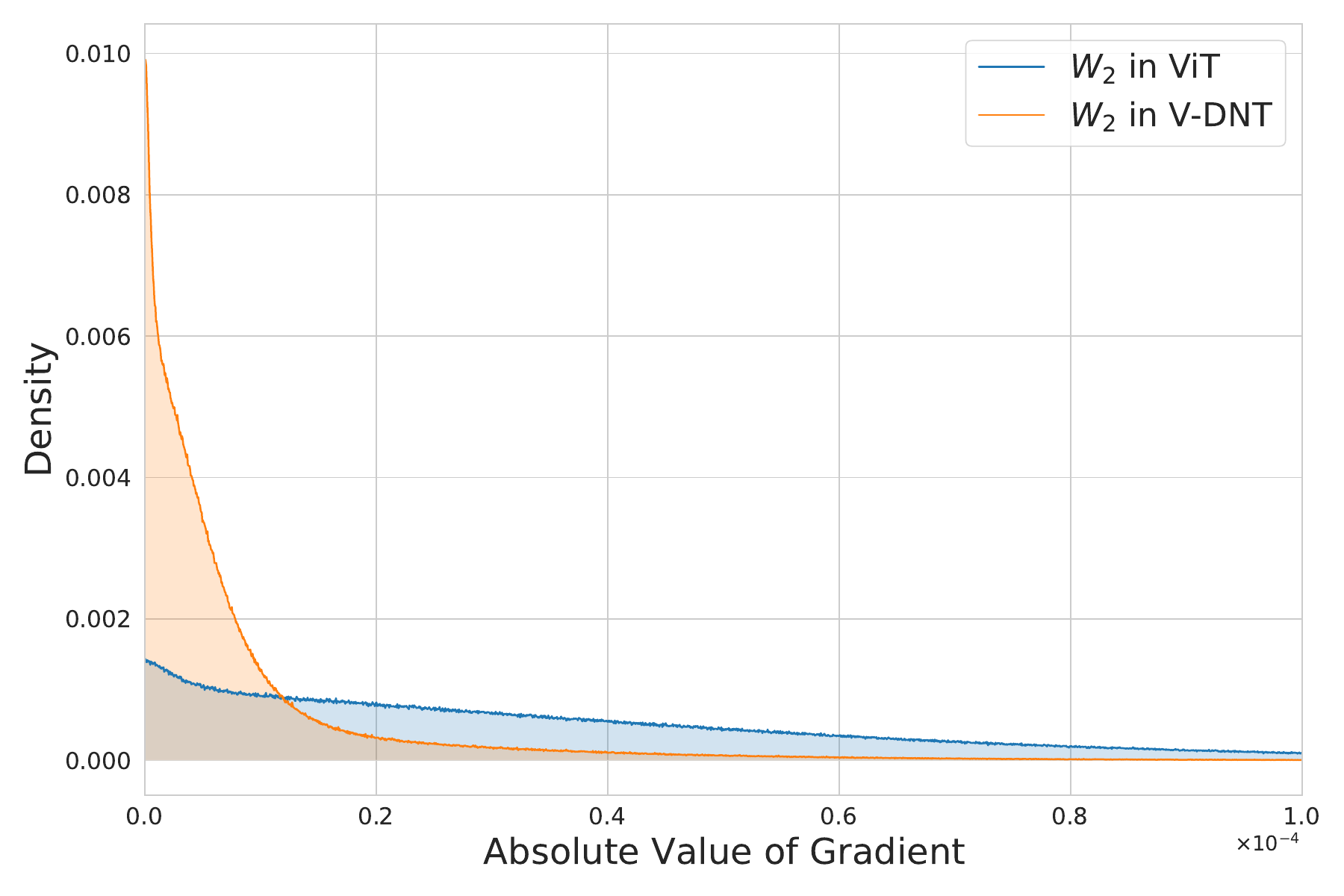}
		\label{chutian2}
	\end{minipage}
	\caption{
    Distributions of the absolute values of the entries in gradients for ViT with PreNorm (marked in \textcolor{blue}{blue}) and our V-DNT (marked in \textcolor{orange}{orange}), where V-DNT denotes the vision variant of our DNT. We observe that the gradients in our V-DNT are typically quite small and well concentrated; whereas the gradient distributions of the standard ViT have a long tail.
    }
	\label{fig:heavy_tail_visualization}
\end{figure}
Naturally, an interesting question 
arises:  
\begin{quote}
\textit{Can Transformers be trained via mSGD to yield performance matched to that is trained via Adam? Or, under what conditions?}
\end{quote}

To answer these questions, 
we need to understand why mSGD typically underperforms Adam when training Transformer.
Previous studies~\citep{heavy_tail_simsekli2019tail,zhangjingzhao_zhang2020adaptive} reveal that the fundamental  
reason lies in the statistical property of the stochastic gradients in Transformer architectures.
Unlike Convolutional Neural Networks (CNNs)~\citep{cnn_lecun1998gradient, resnet_he2016deep} that are trained on tasks like ImageNet, where the entries of the gradients 
are typically small and well-concentrated around their mean, 
the gradients of Transformer typically exhibit heavy-tailed distributions, as shown in blue in  Figure~\ref{fig:heavy_tail_visualization}. This heavy-tailed distribution  
means that  
the amplitudes of the gradient entries span a wide range and thus it is hard to keep step with each others.
Thus, Adam uses a normalized term between the first-order term (i.e. gradients) and the square-root of the second-order term. Owing  
the normalization, Adam is robust to the heavy-tail distribution of the gradients.  
On contrary, 
mSGD directly uses the first-order gradient with momentum to update the weight, and thus it is hard to keep step with each others when updating the weights. 
This explains why Adam has become the 
standard optimizer 
for training Transformer in practice. 
Consequently, an interesting question turns out to be: can we help mSGD to relieve the heavy-tail distribution problem of the gradients in training Transformers? How?

It turns out that \textit{we need to incorporate an effective strategy to mitigate the heavy-tail gradients problem in training the Transformer.} 
Motivated by 
analyzing the Jacobian matrix of different modules, we propose to add or adjust the positions of normalization operations in the Transformer to relieve the heavy-tail gradient problem.  
Roughly speaking, we need to use the properly positioned normalization
operator to amend the Jacobian matrix of $\frac{\partial \by}{\partial \bx}$  largely 
less affected by the weights, the activations, or the joint influence of the weights and the activations.

Visualized results show that our designed architecture, a Deeply Normalized Transformer, termed as DNT, exhibits a more concentrated gradient distribution 
than its counterpart which 
has a heavy-tailed gradient distribution, as illustrated in orange in Figure~\ref{fig:heavy_tail_visualization}. 
In this paper, we provide 
not only the theoretical justification of the properly positioned normalization operator in our DNT but also %
empirical evaluations to further validate that our DNT outperforms its counterparts, \ie, ViT and GPT, on ImageNet classification and OpenWebText tasks. 
Since that the distributions of DNT are concentrated, training it with the vanilla mSGD can yield performance on par with that with Adam optimizer.

To the best of our knowledge, this is the first work to show that using a vanilla mSGD can train a Transformer to achieve performance comparable to that of using Adam---provided that the Transformer architecture is properly modified to mitigates the heavy-tail gradient problem.

\section{Preliminaries}
\label{sec_pre}

This section will provide some preliminaries on high-dimensional random vectors, which enjoy many 
nice properties that are different from their low-dimensional counterparts. Two simple yet useful 
theorems are introduced below. Proofs can be found in Lemma 3.2.4 
of ~\citep{vershynin2018high}.

\begin{theorem}[Concentration of norm]
 Let $\bx$ be an isotropic random vector in $\mathbb{R}^d$. Then, we have
$\mathbb{E} \|\bx\|_2^2 = d.$
Moreover, if $\bx$ and $\by$ are two independent isotropic random vectors, then
$\mathbb{E} \langle \bx, \by \rangle^2 = d.$
\label{theorem:norm_concerntraion} 
\end{theorem}
\begin{theorem}[Almost orthogonality of high-dimensional independent vectors]
Let us normalize the random vectors $\bx$ and $\by$ in Theorem 1, setting
$\overline{\bx} := \frac{\bx}{\|\bx\|_2}$ and $\overline{\by} := \frac{\by}{\|\by\|_2}$, in a high-dimensional space, the independent and isotropic random vectors $\overline{\bx}$  and $\overline{\by}$,  tend to be almost orthogonal,
\label{remark:high_dimension}
\end{theorem}
Theorem~\ref{theorem:norm_concerntraion} establishes that $\|\bx\|_2 \asymp \sqrt{d}$, $\|\by\|_2 \asymp \sqrt{d}$ and $\langle \bx, \by \rangle \asymp \sqrt{d}$ with high probability, which implies that cosine of the angle $\theta$ between two random vectors $\bx$ and $\by$ satisfies 
$\cos(\theta) \asymp \frac{1}{\sqrt{d}}$.
Theorem~\ref{remark:high_dimension} implies that in high-dimensional space (\ie, $d$ is very large), two random vectors are almost orthogonal. 
Thus, given $\bz = \bx + \by$ where $\bx$ and $\by$ are two high-dimensional random vectors, we have $\| \bz \|_2 \asymp  \sqrt{\| \bx \|_2^2 + \| \by \|_2^2}.$

\textbf{Jacobian of normalization.}
Normalization~\citep{batch_normalization_ioffe2015batch, layernorm_ba2016layer,rmsnorm_zhang2019root} is a technique widely used in deep learning. 
It is used to stabilize and accelerate the training process. 
For example, 
LayerNorm is defined as 
$\operatorname{LN}(\bx) = \boldsymbol{\gamma} \odot \frac{\sqrt{d} \by}{ \sqrt{{\| \by \|}_2^2 + \epsilon}}  + \boldsymbol{\beta}, \ \ \text{and} \ \by = \left( \bI - \frac{1}{d} \boldsymbol{1} \boldsymbol{1}^{\top} \right)\bx,$
where $\epsilon >0$ is a smoothing factor,   $d$ is the feature dimension of $\bx$, $\boldsymbol{\gamma}$ and $\boldsymbol{\beta}$ are two learnable  $\mathbb{R}^{d}$ vectors, $\boldsymbol{\gamma}$ and $\boldsymbol{\beta}$ are usually initialized to 1 and 0.
Most recently, some recent LLMs 
\citep{llama2_touvron2023llama, palm_chowdhery2023palm, qwen_team2023qwen, deepseek_v3_liu2024deepseek} uses RMSNorm~\citep{rmsnorm_zhang2019root} to replace LayerNorm, where RMSNorm is defined as:
$\operatorname{RMSN}(\bx) = \boldsymbol{\gamma} \odot  \frac{ \sqrt{d} \bx}{ \sqrt{{\| \bx \|}_2^2 + \epsilon}}.$
Compared to LayerNorm, RMSNorm does not use the centering term and the bias term.

The Jacobian matrix of RMSNorm with respect to $\bx$ is calculated as follows
\begin{small}
\begin{equation*}
\begin{aligned} 
\frac{\partial \operatorname{RMSN}(\bx)}{\partial \bx} &=   \frac{\sqrt{d}}{{ \sqrt{{\| \bx \|}_2^2 + \epsilon}}} \operatorname{diag}(\boldsymbol{\gamma}) \left(\boldsymbol{I}-\frac{\bx \bx^{\top}}{\|\bx\|_{2}^{2} + \epsilon }\right). 
\end{aligned}
\label{equ:jocabian}
\end{equation*}
\end{small}
We use RMSNorm as our default normalization technique when we discuss normalization, but our analysis can be generalized to the other normalization methods. Here we use a numerator layout for all our derivation throughout this paper.

\textbf{Stochastic Gradient Descent (SGD)}~\citep{sgd_robbins1951stochastic} is a classical and fundamental optimization algorithm in machine learning for training models by minimizing their cost functions. However, the vanilla SGD often suffers from slow convergence, especially in complex optimization landscapes with ravines, saddle points, or local minima. To address these limitations, momentum SGD~\citep{nesterov1983method, nestorov_nesterov2013introductory, sutskever2013importance} was introduced as an extension of the basic SGD algorithm.

Momentum SGD~\citep{nesterov1983method, nestorov_nesterov2013introductory, sutskever2013importance} introduces a velocity term $\bmm$ that accumulates gradients over time, \ie, 
\begin{equation*}
\begin{aligned}
\bmm_{t+1} = \mu \bmm_t + \nabla L(\bw_t), \ \ 
\bw_{t+1} = \bw_t - \alpha_t \bmm_{t+1}
\end{aligned}
\end{equation*}
where $\mu \in [0, 1)$ is the momentum coefficient that determines how much of the previous velocity is retained and $\alpha_t$ is the learning rate for the time step $t$.\footnote{Typical values for $\mu$ range from 0.9 to 0.99. In default, for all our experiments, we set 
$\mu$ to 0.90. }
Unlike in the vanilla SGD, mSGD allows the optimization to build up a ``momentum'' in direction of persistent gradient descent, which can effectively dampen the oscillations in high-curvature directions.

\section{Theoretical Justification on Why DNT Can Be Trained with Momentum SGD}
\subsection{Problem 1: What is the root cause of heavy-tail distribution of gradients?}
\label{sec:sgd}

Previous works~\citep{zhangjingzhao_zhang2020adaptive, heavy_tail_simsekli2019tail} have pointed out that a heavy-tailed distribution of the stochastic gradients is a root cause of SGD's poor performance. Here, we would like to 
investigate this issue 
by analyzing the backpropagation in the backward process of the Transformer.

Suppose $\bx^{l+1} = f(\bx^{l})$ and we have obtained $	\frac{\partial \mathcal{L}}{\partial \bx^{l+1}} $ in a backpropagation process,  then we can calculate $	\frac{\partial \mathcal{L}}{\partial \bx^{l}} $ using a numerator layout as
\begin{align*} 
	\frac{\partial \mathcal{L}}{\partial \bx^{l}} = \frac{\partial \mathcal{L}}{\partial \bx^{l+1}}  \frac{\partial {\bx^{l+1}}}{\partial \bx^{l}}, 
\end{align*}
\noindent where $\frac{\partial {\bx^{l+1}}}{\partial \bx^{l}} $ is usually called as the Jacobian matrix.

Having had $\frac{\partial \mathcal{L}}{\partial \bx^{l}}$, 
for any a forward layer with $\bx^{l} = \bW^{l} \bx^{l-1}$, we can compute $\frac{\partial \mathcal{L}}{\partial \bW^{l}}$ as
\begin{align} 
\frac{\partial \mathcal{L}}{\partial \bW^{l}} = \frac{\partial \mathcal{L}}{\partial \bx^{l}} {\bx^{l-1}}^{\top} =  \frac{\partial \mathcal{L}}{\partial \bx^{l+1}} \frac{\partial {\bx^{l+1}}}{\partial \bx^{l}} {\bx^{l-1}}^{\top} .
\label{eq:jacobian_motivation}
\end{align}
From Equation~\ref{eq:jacobian_motivation},
we 
observe that 
the heavy-tail problem in gradients is indeed closely related to the large diversity of the singular values in the Jacobian matrix $\frac{\partial {\bx^{l+1}}}{\partial \bx^{l}}$.
The Jacobian matrix can have highly diverse singular values for several reasons: 1) the weight matrix contains very diverse singular values; 2) the activations span widely, leading to Jacobians with very uneven singular value distributions.
When a matrix has a wide range of singular values (\ie, a very large condition number), it means that the transformation stretches the input 
very differently along different directions. During backpropagation, 
it will cause a heavy-tail problem in the gradients.

\textit{Therefore, one reasonable solution to relieve the heavy-tail problem is to constrain the uneven singular value distribution of the Jacobian matrix via controlling the weight matrix and activations. This is the basic idea 
in our paper.}

\subsection{Problem 2: Mitigate the heavy-tail gradient problem by analyzing the Jacobian matrix}

In this subsection, we will describe how we use different normalizations---adding or adjusting the position of the normalizations---to constrain the Jacobian matrix to relieve the heavy-tail gradient problem. 
\textit{Note that we do not claim that we discover any new normalization methods, instead, we give our understanding on how each normalization affects the Jacobian matrix.} 
We refer the readers 
to~\citep{ngpt_loshchilov2024ngpt, tears_zhu2025transformers, stable_transformer_qi2025stabletransformer, lipsformer_qilipsformer} for more discussions about normalization. 
We will use {{\textcolor{cred}{red},  \textcolor{green}{green},  \textcolor{blue}{blue},  \textcolor{purple}{purple},  \textcolor{magenta}{magenta}}} to denote its relationship with InputNorm, PreNorm, MidNorm, PostNorm and QKNorm, individually.  

\begin{figure}[h]
\tiny
     \centering
     \begin{subfigure}[b]{0.3\textwidth}
        \begin{tikzpicture}[
            box/.style={draw, rounded corners, minimum width=0.6cm, minimum height=0.6cm, thick},
            bloc/.style={draw, 
                minimum width=0.6cm, minimum height=0.6cm, 
                text=red, font=\sffamily, align=center,   
                outer sep=0pt},
            widebox/.style={draw, rounded corners, minimum width=1.0cm, minimum height=0.6cm, thick},
            >=Stealth
          ]
          \node[box] (we) {WE/PE};
          \node[bloc, right=0.4cm of we] (n) {InputNorm};
          \node[box, right=0.4cm of n] (transformer) {Blocks};

          \coordinate (end) at ($(transformer) + (0.8,0)$);
          
          \draw[->] ($(we.west)-(0.5,0)$) -- (we);
          \draw[->] (we) -- (n);
          \draw[->] (n) -- (transformer);
          \draw[->] (transformer) -- (end);
        \end{tikzpicture}
        \caption{InputNorm}
     \end{subfigure}
     \hfill
     \begin{subfigure}[b]{0.3\textwidth}
        \centering
        \begin{tikzpicture}[
            box/.style={draw, minimum width=0.6cm, minimum height=0.6cm, thick},
            bloc/.style={draw, 
                minimum width=0.6cm, minimum height=0.6cm, 
                text=darkgreen, font=\sffamily, align=center,   
                outer sep=0pt},
            circ/.style={draw, circle, minimum size=0.6cm, thick},
            >=Stealth
          ]
          
          \node[bloc] (ln) at (0,0) {PreNorm};
          \node[box, right=0.4cm of ln] (sa) {SA/FFN};
          \node[circ, right=0.4cm of sa] (plus) {$+$};
          
          \coordinate (start) at ($(ln) + (-1.0,0)$);
          \coordinate (end) at ($(plus) + (0.8,0)$);
          
          \draw[->] (start) -- (ln);
          \draw[->] (ln) -- (sa);
          \draw[->] (sa) -- (plus);
          \draw[->] (plus) -- (end);
          
          \path (start) -- (ln) coordinate[pos=0.5] (midpoint);
          \draw[<-] 
            (plus) -- ++(0,0.5) -| (midpoint);
        \end{tikzpicture}
        \caption{PreNorm}
     \end{subfigure}
    \hfill
     	 \begin{subfigure}[{b}]{0.3\textwidth}
        \centering
        \begin{tikzpicture}[
            box/.style={draw, minimum width=0.6cm, minimum height=0.6cm, thick},
            bloc/.style={draw, 
                minimum width=0.6cm, minimum height=0.6cm, 
                text=blue, font=\sffamily, align=center,   
                outer sep=0pt},
            circ/.style={draw, circle, minimum size=0.6cm, thick},
            >=Stealth
          ]
          \node[box] (ln) at (0,0) {SA/FFN};
          \node[bloc, right=0.4cm of ln] (sa) {MidNorm};
          \node[circ, right=0.4cm of sa] (plus) {$+$};
          \coordinate (start) at ($(ln) + (-1.0,0)$);
          \coordinate (end) at ($(plus) + (0.8,0)$);
          \draw[->] (start) -- (ln);
          \draw[->] (ln) -- (sa);
          \draw[->] (sa) -- (plus);
          \draw[->] (plus) -- (end);
          \path (start) -- (ln) coordinate[pos=0.5] (midpoint);
          \draw[thick, red, <-] 
            (plus) -- ++(0,0.5) -| (midpoint);
        \end{tikzpicture}
        \caption{MidNorm}
    \end{subfigure}
\vskip\baselineskip
    \begin{subfigure}[b]{0.49\textwidth}
      \centering
        \begin{tikzpicture}[
            box/.style={draw, minimum width=0.6cm, minimum height=0.6cm, thick},
            bloc/.style={draw, 
                minimum width=0.6cm, minimum height=0.6cm, 
                text=purple, font=\sffamily, align=center,   
                outer sep=0pt},
            circ/.style={draw, circle, minimum size=0.6cm, thick},
            >=Stealth
          ]
          
          \node[box] (ln) at (0,0) {SA/FFN};
          \node[circ, right=0.4cm of ln] (plus) {$+$};
          \node[bloc, right=0.4cm of plus] (sa) {PostNorm};
          
          \coordinate (start) at ($(ln) + (-1.2,0)$);
          \coordinate (end) at ($(sa) + (1.0,0)$);
          
          \draw[->] (start) -- (ln);
          \draw[->] (ln) -- (plus);
          \draw[->] (plus) -- (sa);
          \draw[->] (sa) -- (end);
          
          \path (start) -- (ln) coordinate[pos=0.5] (midpoint);
          \draw[<-] 
            (plus) -- ++(0,0.5) -| (midpoint);
        \end{tikzpicture}
      \caption{PostNorm}
    \end{subfigure}
    \hfill
    \begin{subfigure}[b]{0.49\textwidth}
    	\begin{tikzpicture}[
    		box/.style={draw, minimum width=0.6cm, minimum height=0.6cm, thick},
    		bloc/.style={draw, 
    			minimum width=0.6cm, minimum height=0.6cm, 
    			text=magenta, font=\sffamily, align=center,   
    			outer sep=0pt},
    		circ/.style={draw, circle, minimum size=0.6cm, thick},
    		>=Stealth
    		]
    		\coordinate (start) at (0,0);
    		\coordinate (branch) at (0.7,0);
    		
    		\node[box] (wq) at (1.5,0.7) {$\bW_q$};
    		\node[bloc](nq) at (2.7,0.7) {QKNorm};
    		
    		\node[box] (wk) at (1.5,0) {$\bW_k$};
    		\node[bloc](nk) at (2.7,0.0) {QKNorm};
    		
    		\node[box] (wv) at (1.5,-0.7) {$\bW_v$};
    		
    		\node[circ] (att) at (4.0,0.35) {$\bA$};
    		\node[circ] (mat) at (4.5,-0.7) {MatMul};
    		\coordinate (end) at (5.4,-0.7);
    		
    		\draw[->] (start) -- (branch);
    		
    		\draw[->] (branch) |- (wq);
    		\draw[->] (branch) -- (wk);
    		\draw[->] (branch) |- (wv);
    		
    		\draw[->] (wq) -- (nq);
    		\draw[->] (nq) -- (att);
    		
    		\draw[->] (wk) -- (nk);
    		\draw[->] (nk) -- (att);
    		
    		\draw[->] (wv) -- ($(wv)+(1.5,0)$) |- (mat);
    		
    		\draw[->] (att) -- (mat);
    		
    		\draw[->] (mat) -- (end);
    	\end{tikzpicture}
    	\caption{QKNorm}
    \end{subfigure}
    \caption{Five different normalization methods.
    	The only difference between them is the position of normalization.
    	In (A), WE/PE indicates word embedding and patch embedding.}
    \label{fig:five_norm}
\end{figure}
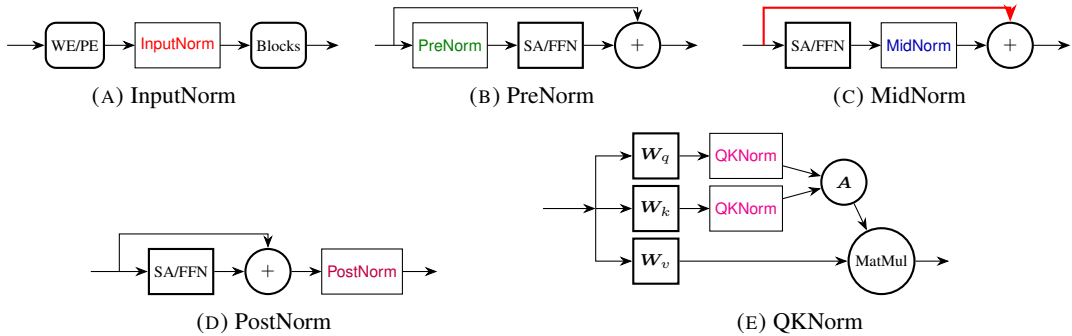

\subsubsection{InputNorm}

\textbf{Definition of InputNorm.} InputNorm in Transformer is defined as the normalization that is applied 
after the first word embedding in NLP or the first patch embedding in vision Transformer. As shown in Figure~\ref{fig:five_norm} (d), InputNorm 
is defined as
\begin{equation}
	\bx^0 = {\Large \textcolor{cred}{ \operatorname{InputNorm}}} (\bh), \ \text{where} \ \  \bh = \operatorname{Embedding}(\bi), \\ 
\end{equation}
where $\bi$ is the input and $\operatorname{Embedding}(\cdot)$ denotes word embedding or patch embedding. 
For a standard residual block in Transformer,  we have the following equation
\begin{align*}
        \bx^{l+1} &=  \bx^{l} + f(\bx^{l} ) = \bx^{l-1} + f(\bx^{l-1}) + f(\bx^{l} ) = \bx^{0} + f(\bx^0) + f(\bx^1) + \dots + f(\bx^{l-1}) + f(\bx^{l}).
\end{align*}
Each $\bx^{l}$ will be the input into some modules, such as normalization, self-attention and feed-forward layers. The Jacobian matrices of some modules are sensitive to the norm of the input, such as LayerNorm and the dot-product self-attention.

Under the assumption that random vectors are almost orthogonal in high dimension, we have that
\begin{align}
  \| \bx^{l+1} \|_2 \asymp \sqrt{( \eqnmarkbox[red]{Psi2}{\| \bx^{0} \|_2^2} + \| f(\bx^{0}) \|_2^2 + \cdots +\| f(\bx^{l}) \|_2^2)}.
 	\label{eq:inputnorm}
\end{align}

\begin{proposition}[Effect of norm of input embedding on gradients]
\label{proposition_inputnorm}
In a high-dimensional settings when all parameters and activations are high-dimensional, if the norm of $\bx^{0}$ is very large, it will lead to gradient vanishing problem in all subsequent layers, provided that InputNorm is not used. 
\end{proposition}

It means that if the norm of $\bx^{0}$ is large, then the norm  $\| \bx^{l+1} \|_2$ in each layer will also be large. 
If $\| \bx^{l+1} \|_2$ is the input into a normalization layer,  according to the Jacobian equation of normalization $\frac{\partial \operatorname{RMSN}(\bx^{l+1})}{\partial \bx^{l+1}} =   \frac{\sqrt{d}}{{ \sqrt{{\| \bx^{l+1} \|}_2^2 + \epsilon}}} \operatorname{diag}(\boldsymbol{\gamma}) \left(\boldsymbol{I}-\frac{\bx^{l+1} {\bx^{l+1}}^{\top}}{\|\bx^{l+1}\|_{2}^{2} + \epsilon }\right)$, the gradient flow in each layer will be significantly 
affected by the norm of $\bx^{0}$. 
Thus, we need to constrain the norm of $\bx^{0}$ before it is used as the input into the following layer.

\begin{remark}
 The norm of $\bx^0$ has a large influence of the gradient flow of the subsequent layers. If it is very large, it will lead to gradient vanishing, and if it is very small, it may lead to gradient exploding. Meanwhile, the network is also sensitive to the change of the norm of $\bx^0$.
\end{remark}

\subsubsection{PreNorm}

\textbf{Definition of PreNorm.}  A PreNorm in Transformer is defined as the normalization that is applied  
before the self-attention or the feed-forward components. As shown in Figure~\ref{fig:five_norm} (b), PreNorm is 
defined as
\begin{align}
\bY = \operatorname{Self-Attention}(\bX'), \text{where} \ \bx' = \operatorname{PreNorm}(\bx).
\label{eq:def_pn}
\end{align}
A  single-head self-attention is 
defined as
\begin{equation*}
	\bY = \bW_v \bX \bA,
\end{equation*}
\noindent where $\bP = \bX^{\top} {\bW_q}^{\top} {\bW_k} {\bX},  \quad \bA = \operatorname{softmax}(\frac{\bP}{\sqrt{d_q}})$, 
$\bA$ is called as the attention matrix and $\frac{\bP}{\sqrt{d_q}}$ is called as the logit,  $\bA \in \mathcal{R}^{n\times n}, \bX \in \mathcal{R}^{d\times n}, \bW_q \in \mathcal{R}^{d_q\times d}, \bW_k \in \mathcal{R}^{d_q\times d}, \bW_v \in \mathcal{R}^{d_v\times d}$. 
Here, our goal is to calculate $\frac{\partial \text{vec}(\bY)}{\partial \text{vec}(\bX)} $.

By vectorization of $ \bY = \bW_v \bX \bA$, we have 
\begin{equation*}
	\partial {\text{vec}(\bY)} =  (\bA^\top \otimes \bW_v) \partial  \text{vec}(\bX)  + (\bI_n \otimes \bW_v\bX) \partial  \text{vec}(\bA).
\end{equation*}
Bringing together all these terms, we 
have that
{\footnotesize
	\begin{equation}
		\frac{\partial \text{vec}(\bY)}{\partial \text{vec}(\bX)} = (\bA^\top \otimes \bW_v) + (\bI_n \otimes \bW_v\eqnmarkbox[darkgreen]{Psi2}{\bX})  \frac{\bJ}{\sqrt{d_q}}  \left( 
		(\eqnmarkbox[darkgreen]{Psi2}{\bX^{\top}}{\bW_k}^{\top}{\bW_q} \otimes \bI_n)\bC + (\bI_n \otimes \eqnmarkbox[darkgreen]{Psi2}{\bX^{\top}}{\bW_q}^{\top}{\bW_k})
		\right).
	\label{eq:prenorm}
	\end{equation}
}
For 
simplicity, we 
denote
$$
\bJ  = \text{blockdiag}\left(\text{diag}(\bA_{:,1}) - \bA_{:,1} \bA_{:,1}^\top, \dots, \text{diag}(\bA_{:,n}) - \bA_{:,n} \bA_{:,n}^\top\right).
$$
A detailed derivation process can be found in~\citep{qi_taming_transformer_qitaming}. 
\textit{Different from~\citep{qi_taming_transformer_qitaming}, This paper is 
to analyze the influence of $\bX$, rather than 
$\bW_q^{\top} \bW_k$} in the self-attention module. 
According to the Jacobian matrix in Equation~\ref{eq:prenorm}, we have the following proposition.

\begin{proposition}[PreNorm can stabilize the gradient in self-attention module.]
\label{proposition_prenorm}
If $\bX' = \bg \odot \bX$ with each column in $\bx'_i$ having a magnitude $\bg_i$, according to Equation~\ref{eq:def_pn}, with the same $\bW_q, \bW_k, \bW_v$, 
$\bY = \operatorname{Self-Attention}(\bX)$ and 
$\bY'= \operatorname{Self-Attention}(\bX')$. 
Then we have
$\frac{\partial \text{vec}(\bY)}{\partial \text{vec}(\bX)} = 	\frac{\partial \text{vec}(\bY')}{\partial \text{vec}(\bX')}$.
\end{proposition}

According to Proposition~\ref{proposition_prenorm}, we have the following remark.
\begin{remark}
PreNorm will guarantee that the norms of vectors $\bX$ which is the input to the self-attention layers are 
in a relatively stable range of norms. 
According to Equation~\ref{eq:prenorm}, we 
see that if these norms of $\bX$ are relatively stable, then the Jacobian matrix will also be stable relative to $\bX$. 
Meanwhile, since the gradients of $\bW_q, \bW_k, \bW_v$ are directly relative to $\bX$, a stable $\bX$ will guarantee that $\bW_q, \bW_k, \bW_v$ obtains relatively stable gradients.
\end{remark}

\subsubsection{MidNorm}

\textbf{Definition of MidNorm.} A MidNorm in Transformer is defined as the normalization that is applied 
after the self-attention and feed-forward components and meanwhile before the residual shortcut. 
As shown in Figure~\ref{fig:five_norm} (b), MidNorm 
is defined as
\begin{equation} 
\by = \operatorname{MidNorm}(\bz), \text{where} \ \ \bz = \bW_2 \operatorname{ReLU}(\bW_1 \bx).
\label{eq:def_midnorm}
\end{equation}
In the self-attention, $\bW_v$ and $\bW_o$ can be seen as similar function as $\bW_1$ and $\bW_2$ in FFN. If we only use single-head attention, then we have $ \bz = \bW_o \bW_v \bx$.

The Jacobian matrix of an FFN can be computed as:
$
\boldsymbol{J}_{\bz}(\bx) = \frac{\partial \operatorname{FFN}(\bx; \bW_1, \bW_2)}{\partial \bx} =    {\bW_2} \operatorname{diag} { \left( \boldsymbol{1}\left({ \bW_1  \bx>\bold{0}}\right) \right)} {\bW_1} .
$
The Jacobian matrix of an RMSNorm layer is
$
		\frac{\partial\by}{\partial \bz} =   \frac{\sqrt{d}}{{{{\| \bz \|}_2 }}} \operatorname{diag}(\boldsymbol{\gamma})  \left(\boldsymbol{I}-\frac{\bz \bz^{\top}}{\|\bz\|_{2}^{2} }\right). 
$ 
The joint Jacobian matrix of an FFN and an RMSNorm is
\begin{equation}
	\begin{aligned} 
		\frac{\partial\by}{\partial \bx} &=  \frac{\partial\by}{\partial \bz}  \frac{\partial\bz}{\partial \bx} =  \sqrt{d}  \operatorname{diag}(\boldsymbol{\gamma})  \left(\boldsymbol{I}-\frac{\bz \bz^{\top}}{\|\bz\|_{2}^{2} }\right)   \eqnmarkbox[blue]{Psi2}{ \frac{{\bW_2} \operatorname{diag} { \left( \boldsymbol{1}\left({ \bW_1  \bx>\bold{0}}\right) \right)} {\bW_1}}{{\|  \bW_2 \operatorname{ReLU}(\bW_1 \bx) \|}_2} }. 
	\end{aligned}
	\label{eq:midnorm}
\end{equation}

\begin{proposition}[The effect 
of MidNorm]
\label{proposition_midnorm}
Let $\bW = \frac{{\bW_2} \operatorname{diag} { \left( \boldsymbol{1}\left({ \bW_1  \bx>\bold{0}}\right) \right)} {\bW_1}}{{\|  \bW_2 \operatorname{ReLU}(\bW_1 \bx) \|}_2},$ in a high-dimensional  settings when $\bW_1$, $\bW_2$ and $\bx$ are high-dimensional and random, the singular values of $\ \bW$ will be only related to the shape of $\bW_1$ and $\bW_2$, and will be independent to the magnitude of $\bW_1$ and $\bW_2$.
\end{proposition}

According to Proposition~\ref{proposition_midnorm}, we have the following remark.
\begin{remark}
MidNorm can effectively guarantee that the norms of $\ \bW_1$, $\bW_2$, $\bW_v$, and $\bW_o$ will not affect the Jacobian matrix as shown in Equation~\ref{eq:midnorm}. It means that even the magnitudes of these weight matrices are very large, it will not magnify the gradient.
\end{remark}

\subsubsection{PostNorm}

\textbf{Definition of PostNorm.} A PostNorm in Transformer is defined as the normalization that is applied 
after the residual block. As shown in Figure~\ref{fig:five_norm} (d), PostNorm is 
defined as
\begin{equation} 
	\bx^{l+1} = \operatorname{PostNorm}(\bz^{l+1}), \text{where} \ \ \bz^{l+1} = \bx^l + f(\bx^l; \bW^{l+1}).
\label{eq:def_postnorm}
\end{equation}

\begin{proposition}[PostNorm is sensitive to the vector norm of activation]
\label{proposition_postnorm}
If $\bz^{l+1}$ in Equation~\ref{eq:def_postnorm} is very large, then it will significantly 
decrease the gradient.
\end{proposition}

\begin{proof}
From Equation~\ref{eq:def_postnorm}, we have $\frac{\partial \bx^{l+1}}{\partial \bz^{l+1}} = \frac{\sqrt{d}}{\|\bz^{l+1} \|} (\bI - \frac{\bz^{l+1} {\bz^{l+1}}^{\top}}{\| \bz^{l+1} \|_2^2})$. 
If $\|\bz^{l+1} \|$ is very large, according to $\frac{\partial L}{\partial \bz^{l+1}}=\frac{\partial L}{\partial \bx^{l+1}}  \frac{\partial \bx^{l+1}}{\partial \bz^{l+1}}$, we have that 
the gradient of $\frac{\partial L}{\partial \bz^{l+1}}$ will be largely decreased.
\end{proof}

In a classical Transformer~\citep{transformer_vaswani2017attention}, if $f(\bx; \bW_1,\bW_2) =  \bW_2 \operatorname{ReLU}(\bW_1 \bx)$, along with the training process, $\sigma(\bW_1)$ and $\sigma(\bW_2)$ will usually become too large. 
In this way, $\| f(\bx^l; \bW^{l+1}) \|_2$ will be very large,  it means $\bz^{l+1}$ in Equation~\ref{eq:def_postnorm} will be very large. 
Therefore, we have that PostNorm under this circumstance  will lead to gradient vanishing.

\begin{remark}
We need to be very careful when using PostNorm, we must ensure that the norm of the input vector to PostNorm is within a reasonable range, otherwise the 
network is likely to cause a 
gradient vanishing when $\bz^{l}$ being very large or a gradient exploding $\bz^{l}$ when $\bz^{l}$ being very small. 
\end{remark}

\subsubsection{QKNorm}

\textbf{Definition of QKNorm.} A QKNorm~\citep{qk_norm_henry2020query} in Transformer is defined as the normalization that is applied 
on queries and keys in the self-attention block. As shown in Figure~\ref{fig:five_norm} (e), 
self-attention with QKNorm~\citep{scaling22b_dehghani2023scaling} is defined as
\begin{equation}
	\bY = \bW_v \bX \bA', \text{where}\ \  \bA' = \operatorname{softmax}(\frac{\bP'}{\sqrt{d_h}}), \ \bP' = \bQ'^{\top} \bK',
\label{eq:def_qknorm}
\end{equation}
\noindent in which $\bq'_i$ and $\bk'_i$ are the $i$-th column and the $j$-th column in $\bQ'$ and $\bK'$, individually, and we have
\begin{equation}
\small
	\centering
	\begin{aligned} 
		\bq'_i &= \operatorname{QKNorm}({\bW_q} \bx_i ) = \boldsymbol{\gamma}_q \odot   \frac{ \sqrt{d_h} {\bW_q} \bx_i }{ {\| {\bW_q} \bx_i  \|}_2 } =  \sqrt{d_h} \operatorname{diag}( \boldsymbol{\gamma} _q) \eqnmarkbox[magenta]{Psi2}{\frac{ {\bW_q} \bx_i  }{ {\| {\bW_q} \bx_i  \|}_2}}, \\
		\bk'_j &= \operatorname{QKNorm}({\bW_k} \bx_j ) = \boldsymbol{\gamma}_k \odot   \frac{ \sqrt{d_h} {\bW_k} \bx_j }{ {\| {\bW_k} \bx_j  \|}_2 } =  \sqrt{d_h} \operatorname{diag}( \boldsymbol{\gamma} _k) \eqnmarkbox[magenta]{Psi2}{\frac{ {\bW_k} \bx_j  }{ {\| {\bW_k} \bx_j  \|}_2}},
	\end{aligned}
\label{eq:qknorm}
\end{equation}
\noindent where $d_h$ is the head dimension. To facilitate the 
derivation, we 
denote $\bQ = \bW_q \bX$ and $\bK = \bW_k \bX$ as before, and use $\bq_i$ and $\bk_j$ to denote the $i$-th column and the $j$-th column in $\bQ$ and $\bK$, individually. Thus, we 
have that $\bq'_i = \operatorname{QKNorm}(\bq_i)$ and $\bk'_j = \operatorname{QKNorm}(\bk_j)$.

Moreover, we have that $P'_{ij} = {\bq'}_i^{\top} \bk'_j$, where $P'_{ij}$ is a scalar, and the gradient is computed as follows,
\begin{equation}
	\begin{aligned}
	\frac{\partial P'_{ij}}{\partial \bx} &=  {\bk'}_j^{\top} \frac{\partial \bq'_{i}}{\partial \bx} +  {\bq'}_i^{\top} \frac{\partial \bk'_{j}}{\partial \bx} \\
	&= \sqrt{d_h} \operatorname{diag}( \boldsymbol{\gamma} _q)  {\bk'}_j^{\top}  (\bI - \frac{\bq'_i \bq'_i}{\|  \bq'_i\|_2^2}) \frac{\bW_q}{\| {\bW_q} \bx_i  \|_2}  +  \sqrt{d_h} \operatorname{diag}( \boldsymbol{\gamma} _k)  {\bq'}_i^{\top}  (\bI - \frac{\bk'_j \bk'_j}{\|  \bk'_j\|_2^2}) \frac{\bW_k}{\| {\bW_k} \bx_i  \|_2}  .
	\end{aligned} 
\label{eq:qknorm_deriviation}
\end{equation}

\begin{proposition}[Effect of QKNorm] 
\label{proposition_qknorm}
In a high-dimensional settings, \ie, when all $\bW_q$, $\bW_k$ and $\bx$ are high-dimensional and random, in  Equation~\ref{eq:qknorm_deriviation}, the gradient term of $\frac{\partial P'_{ij}}{\partial \bx}$ is independent of the magnitudes $\bW_q$ and $\bW_k$.
\end{proposition}

According to the Proposition, we have the following remark. 
\begin{remark}
QKNorm can mitigate the joint effect of $\bW_q^{\top} \bW_k$ to the gradient of the self-attention layer.  The fast increasing of the singular values of $\bW_q^{\top} \bW_k$ has been 
revealed to be a reason leading to 
model crash. 
Our analysis shows that QKNorm can effectively mitigate the 
reason to cause model crash brought by $\bW_q^{\top} \bW_k$.
\end{remark}

Though QKNorm can mitigate the problem brought by $\bW_q^{\top} \bW_k$, it cannot fully replace the 
role of PreNorm because PreNorm can jointly deal with the problem of $\bW_q, \bW_k$ and $\bW_v$ because the gradient of $\bW_v$ is also affected by the value of $\bX$.

\subsection{DNT: A network that can relieve the heavy-tail gradient problem}

\begin{figure}[h]
\centering
\tiny

\begin{tikzpicture}[
    box/.style={draw, minimum width=0.8cm, minimum height=0.6cm, thick},
    inputnorm/.style={draw, minimum width=0.8cm, minimum height=0.6cm, 
                text=red, font=\sffamily, align=center,   
                outer sep=0pt},
    prenorm/.style={draw, 
                minimum width=0.8cm, minimum height=0.6cm, 
                text=darkgreen, font=\sffamily, align=center,   
                outer sep=0.0pt},
    prenorm2/.style={draw, 
                minimum width=0.8cm, minimum height=0.6cm, 
                text=darkgreen, font=\sffamily, align=center,   
                outer sep=0.0pt},
    prenorm3/.style={dashed, draw,  rounded corners=3mm,
                minimum width=0.8cm, minimum height=0.6cm, 
                text=darkgreen, font=\sffamily, align=center,   
                outer sep=0.0pt},    
    midnorm/.style={draw, 
                minimum width=0.8cm, minimum height=0.6cm, 
                text=blue, font=\sffamily, align=center,   
                outer sep=0pt},
    qknorm/.style={draw, 
                minimum width=0.8cm, minimum height=0.6cm, 
                text=magenta, font=\sffamily, align=center,   
                outer sep=0pt},
    bigbox/.style={draw, minimum width=10.3cm, minimum height=1.7cm, thick},
    circ/.style={draw, circle, minimum size=0.6cm, thick},
    >=Stealth
  ]

  \node[box] (we) at (0,0) {WE/PE};
  \node[inputnorm, right=0.4cm of we] (n1) {InputNorm};
  
  \node[bigbox, right=0.2cm of n1, anchor=west] (enclosure) {};
  \node[anchor=south east] at (enclosure.south east) {\large N};
  
  \node[prenorm] (n2) at ($(enclosure.west) + (0.8,-0.0)$) {PreNorm};
  \node[qknorm, right=0.4cm of n2] (sa) {SA with\\QKNorm};
  \node[midnorm, right=0.4cm of sa] (n3) {MidNorm};
  \node[circ, right=0.4cm of n3] (plus1) {$+$};
  
  \node[prenorm3, right=0.4cm of plus1] (n4) {PreNorm};
  \node[box, right=0.4cm of n4] (ffn) {FFN};
  \node[midnorm, right=0.4cm of ffn] (n5) {MidNorm};
  \node[circ, right=0.4cm of n5] (plus2) {$+$};
  
  \draw[->] (n1) --  (n2);

  \draw[->] ++(-1.0,0)  -- (we);
  \draw[->] (we) -- (n1);

  \draw[->] (n2) -- (sa);
  \draw[->] (sa) -- (n3);
  \draw[->] (n3) -- (plus1);
  \draw[->] (plus1) -- (n4);
  \draw[->] (n4) -- (ffn);
  \draw[->] (ffn) -- (n5);
  \draw[->] (n5) -- (plus2);
  \draw[->] (plus2) -- ++(1.0,0);
  
  \path (n1) -- (n2) coordinate[pos=0.7] (midpoint);
   \draw[thick, red, <-] 
            (plus1) -- ++(0,0.7) -| (midpoint);

  \path (plus1) -- (n4) coordinate[pos=0.5] (midpoint);
  \draw[thick, red, <-] 
            (plus2) -- ++(0,0.7) -| (midpoint);

\end{tikzpicture}

\caption{DNT architecture. The second PreNorm marked with dashed and rounded corners is optional. By default, we do not use the second PreNorm.}
\label{fig:heat_net}
\end{figure}
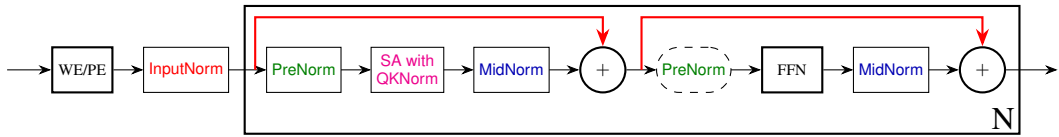

Having 
analyzed the 
effects of different normalizations, 
we use four types of normalizations, including InputNorm, PreNorm, MidNorm and QKNorm, except for PostNorm. 
\textit{The reason why we do not use PostNorm is 
that it may bring in some training problem.} Finally, we illustrate our DNT 
in Figure~\ref{fig:heat_net}.

Our DNT model commences with word embeddings or patch encodings (WE/PE). Then, the initial representations undergo the InputNorm processing, establishing the normalized embeddings for subsequent operations.
The core transformer block consists of $N$ blocks.
In each block, prior to self-attention, a PreNorm is applied, followed by a self-attention 
augmented with query-key normalization (\ie, QKNorm). Subsequently, a MidNorm  processes the attention outputs before integrating via a residual connection. 
In the second sub-block, 
a \textit{selective PreNorm} precedes the feed-forward network (FFN), after FFN, there is a MidNorm, and a final residual connection completes the information flow. 
This entire structure is replicated $N$ times to form the complete network.

\begin{figure}[h]
	\centering
	\tiny
	\begin{subfigure}[b]{0.49\textwidth}
		\centering
		\begin{tikzpicture}[
			box/.style={draw, rounded corners, minimum width=0.8cm, minimum height=0.6cm, thick},
			inputnorm/.style={draw, 
				minimum width=1.2cm, minimum height=0.6cm, 
				text=red, font=\sffamily, align=center,   
				outer sep=0pt},
			>=Stealth
			]
			
			\node[inputnorm] (input) {InputNorm};
			\node[box, right=2.5cm of input] (weights) {$\bW_q, \bW_k, \bW_v, \bW_o, \bW_1, \bW_2$};

			\node[box] (features) at ($(input)!0.5!(weights) + (-0.5,0.6)$) {$X^0, X^1, \ldots, X^L$};
			
			\draw[->] (input) -- (weights);
			
		\end{tikzpicture}
		
		\caption{Influence of InputNorm.}
	\end{subfigure}
	\hfill
	\begin{subfigure}[b]{0.49\textwidth}
		\centering
		\begin{tikzpicture}[
			box/.style={draw, rounded corners, minimum width=0.8cm, minimum height=0.6cm, thick},
			prenorm/.style={draw, 
				minimum width=1.2cm, minimum height=0.6cm, 
				text=darkgreen, font=\sffamily, align=center,   
				outer sep=0pt},
			>=Stealth
			]
			\node[prenorm] (input) {PreNorm};
			\node[box, right=2.5cm of input] (weights) {$\bW_q, \bW_k, \bW_v$};
			
			\node[box] (features) at ($(input)!0.5!(weights) + (-0.2,0.6)$) {$X^0, X^1, \ldots, X^L$};
			
			\draw[->] (input) -- (weights);
			
		\end{tikzpicture}
		\caption{Influence of PreNorm.}
	\end{subfigure}
	\vskip\baselineskip
	\begin{subfigure}[b]{0.49\textwidth}
		\centering
		\begin{tikzpicture}[
			box/.style={draw, rounded corners, minimum width=0.8cm, minimum height=0.6cm, thick},
			midnorm/.style={draw, 
				minimum width=1.2cm, minimum height=0.6cm, 
				text=blue, font=\sffamily, align=center,   
				outer sep=0pt},
			>=Stealth
			]
			
			\node[midnorm] (input) {MidNorm};
			\node[box, right=2.5cm of input] (weights) {$\bW_v, \bW_o, \bW_1, \bW_2$};
			
			\node[box] (features) at ($(input)!0.5!(weights) + (-0.2,0.6)$) {$X^0, X^1, \ldots, X^L$};
			
			\draw[->] (input) -- (weights);
			
		\end{tikzpicture}
		\caption{Influence of MidNorm.}
	\end{subfigure}
	\hfill
	\begin{subfigure}[b]{0.49\textwidth}
		\centering
		\begin{tikzpicture}[
			box/.style={draw, rounded corners, minimum width=0.8cm, minimum height=0.6cm, thick},
			qknorm/.style={draw, 
				minimum width=1.2cm, minimum height=0.6cm, 
				text=magenta, font=\sffamily, align=center,   
				outer sep=0pt},
			>=Stealth
			]
			
			\node[qknorm] (input) {QKNorm};
			\node[box, right=2.5cm of input] (weights) {$\bW_q, \bW_k$};
			
			\node[box] (features) at ($(input)!0.5!(weights) + (-0.0,0.6)$) {$\bQ,\bK$};
			
			\draw[->] (input) -- (weights);
			
		\end{tikzpicture}
		\caption{Influence of QKNorm.}
	\end{subfigure}
	\caption{Influence of different normalizations. For instance, InputNorm stabilizes $\bW_q, \bW_k, \bW_v, \bW_o, \bW_1, \bW_2$ by constraining $X^0, X^1, \ldots, X^L$.}
	\label{fig:four_diagrams}
\end{figure}
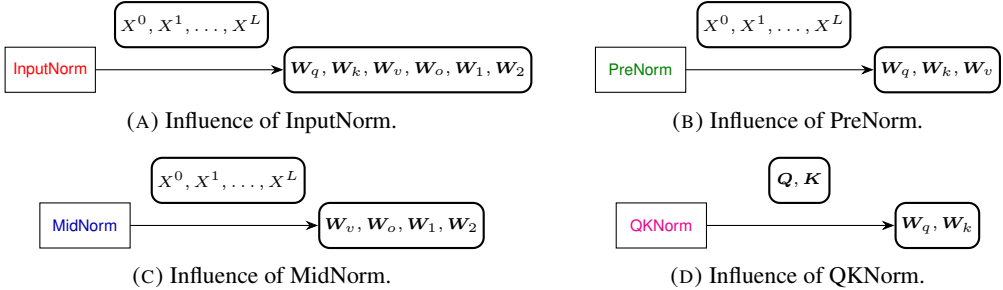

We visualize the effects 
of each normalization in our DNT network in Figure~\ref{fig:four_diagrams}. According to 
the analysis mentioned above, our DNT model has the following advantages: a) the magnitude of $\bx^0$ will 
significantly affect the gradient of each layer in the Transformer, but we introduce InputNorm 
to resolve the influence of $\bx^0$; b) PreNorm can constrain the norm of each column in activations $\bX$ in each timestep, and thus amend 
the Jacobian matrix of self-attention to 
not be significantly 
affected by the magnitude of $\bX$; c) MidNorm will amend 
the Jacobian matrix of each sub-block (\ie, the sub-block with self-attention and the sub-block with FFN) in our DNT 
to not be affected by the magnitude of $\bW_1$, $\bW_2$, $\bW_v$ and $\bW_o$; d) QKNorm  can relieve or even remove the influence of the magnitude of $\bW_q$  and $\bW_k$ on the Jacobian matrix of self-attention, and thus reduce the risk of problems, such as rank collapse~\citep{rank_collapse_noci2022signal}, entropy collapse~\citep{stabilizing_transformer_zhai2023stabilizing}, or spectral energy concentration~\citep{qi_taming_transformer_qitaming} caused by $\bW_q^{\top} \bW_k$.

In DNT, we use four different types of normalizations. We observe that nGPT \citep{ngpt_loshchilov2024ngpt} also uses  
some of the normalizations mentioned above. Here, we would like to emphasize the differences between DNT and nGPT that: a) DNT provides theoretical justifications for each normalization in different position; b) DNT uses InputNorm rather than PostNorm, whereas nGPT use many PostNorms but not InputNorm; c) nGPT normalizes the activations or the weights into spheres, whereas DNT only normalizes the activations but does not requires activations on spheres.

We term our model as Deeply Normalized Transformer (DNT for short), because it is designed by properly adding or positioning normalization operators in the conventional Transformer. 
For vision problem, we term it as V-DNT, and for language problem, we term it as L-DNT. The key difference between V-DNT and L-DNT is that V-DNT uses patch embedding, but L-DNT uses word embedding and mask for attention computation.

\begin{figure}[htbp]
	\centering
	\begin{minipage}{0.495\linewidth}
		\centering
		\includegraphics[width=0.99\linewidth,height=0.8\linewidth]{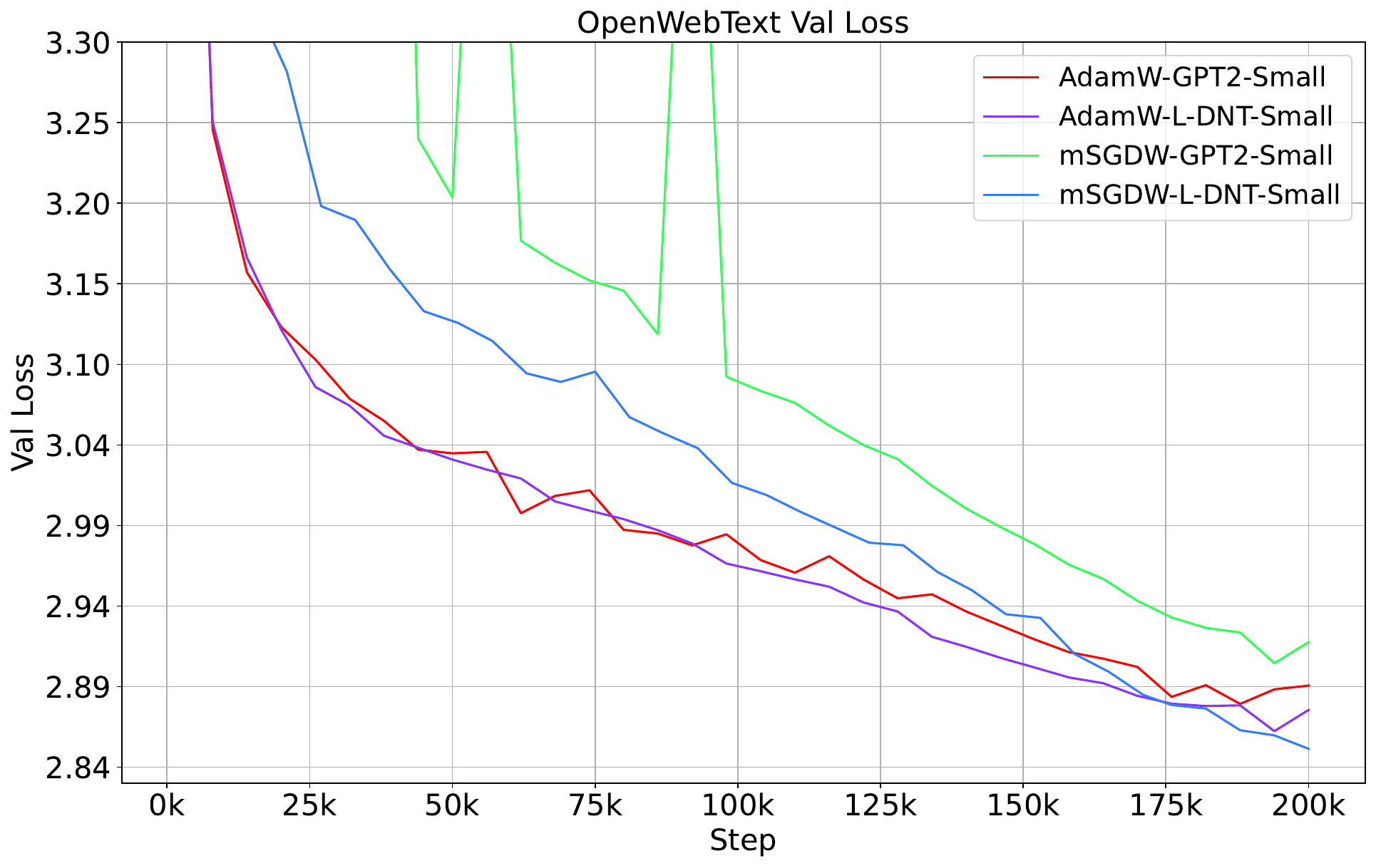}
		\label{chutian1}
	\end{minipage}
	\begin{minipage}{0.495\linewidth}
		\centering
		\includegraphics[width=0.99\linewidth,height=0.8\linewidth]{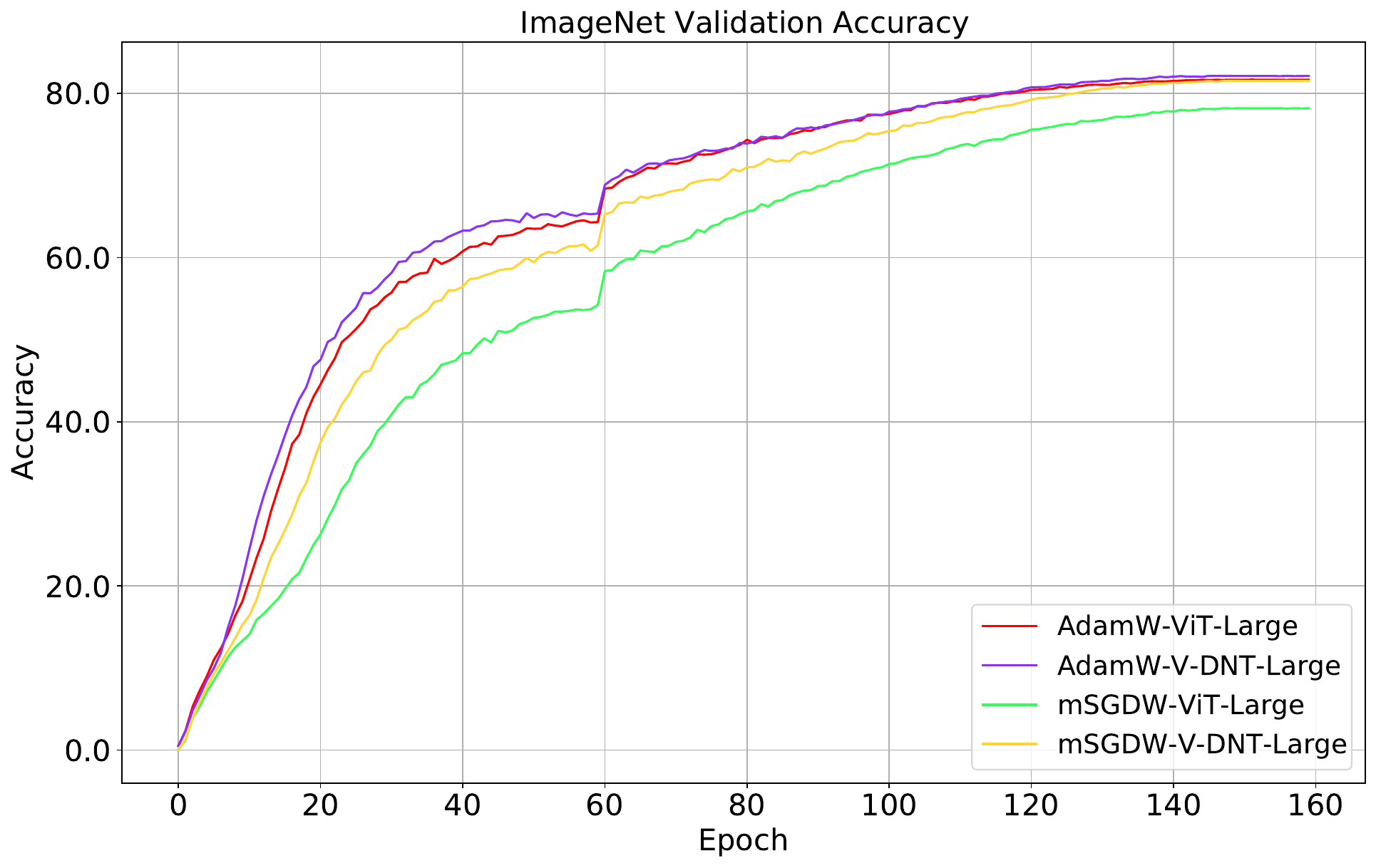}
		\label{chutian2}
	\end{minipage}
	\caption{Validation loss (Left ) on OpenWebText and recognition accuracy (Right) on ImageNet. We compare L-DNT-Small (124M) to GPT2-Small (124M), and V-DNT-Large (307M) to ViT-Large (307M).
    By effectively relieving the heavy-tail gradient problem, our DNT network trained with naive  mSGDW can achieve competitive performance to AdamW (Val loss \textit{ 2.849 vs. 2.863} on OpenWebText, Acc \textit{81. 5\% vs. 82. 1\%} on ImageNet). However, in classical Transformer with PreNorm, the performance of  mSGDW under-performs AdamW significantly  (Val loss \textit{2.906 vs 2.867} on OpenWebText, Acc \textit{78.2\% vs 81.7\%} on ImageNet). See Appendix H for the training parameters.}
\label{exp:lheat_and_vheat_sgd_adam}
\end{figure}

\section{Experiments}

We conducted experiments with two popular Transformer architectures: Vision Transformer (ViT) and Generative Pretrained Transformer (GPT). Our implementation leverages established repositories: timm~\citep{timm_rw2019timm} for ViT and nanoGPT~\citep{nanogpt_Karpathy2022} for GPT models.
 For ViT experiments, we utilized two model scales: ViT-Large (307M parameters) and ViT-Huge (632M), following the configurations described in \citep{vit_dosovitskiy2020image}. The data augmentation strategy aligns with \citep{adan_xie2024adan} to ensure fair comparison with previously reported results.
For GPT experiments, we employed the nanoGPT implementation focusing on GPT2-Small (124M) and GPT2-Large (774M) variants due to computational constraints. The results of
our baselines align with previous work, including Sophia~\citep{sophia_liu2023sophia} on OpenWebText and MAE~\citep{mae_he2022masked} on ImageNet.
Training was conducted using PyTorch~\citep{pytorch_paszke2019pytorch} with bfloat16 precision on A800 GPUs, employing a cosine learning rate schedule.

\subsection{Visualization of gradients of DNT and Transformer with PreNorm}

To visually compare the 
standard Transformer with our DNT, we visualized the gradients of different weights, 
including $\bW_q, \bW_k, \bW_v, \bW_o, \bW_1, \bW_2$.  We chose the early checkpoints of the model training for visualization, but we found that the same phenomenon is  
also presented in the middle and later stages of the model training.

The visualization is shown in Figure~\ref{fig:heavy_tail_visualization}, we can see that, 
DNT network can well relieve the problem of heavy tail of gradient distribution. For instance, in the Transformer, the absolute values of gradients almost distribute even across $[0, 10^{-4}]$, but the absolute values of gradients in DNT mainly concentrate around $[0, 10^{-5}]$.

\subsection{mSGDW achieves performance on par 
with AdamW.}
We also give a quantitative comparison of the standard Transformer and DNT trained with Adam and mSGD on OpenWebText and ImageNet in Table~\ref{tab:experimental_results}. 
We can see that training our 
DNT model 
via mSGDW achieves a similar result to that is 
trained with AdamW. 
We can also see that using mSGDW to 
train our DNT model greatly 
outperforms the performance of using mSGDW to train 
the standard Transformer. 
In Figure~\ref{exp:lheat_and_vheat_sgd_adam}, we display the validation loss on OpenWebText and the training accuracy on ImageNet along with the training process. Note that we didn't tune the learning rate too much. We just followed the learning rate settings of the previous works~\cite{nanogpt_Karpathy2022,sophia_liu2023sophia}.  
We believe tuning learning rate will bring in some differences. But overall, DNT network can enable mSGDW compete with AdamW.
\begin{table}[ht]
\small
	\centering
	\caption{Quantitative comparison of standard ViT/GPT2 and V-DNT/L-DNT trained with AdamW and mSGDW on OpenWebText and ImageNet. Results on ImageNet is based on 150 epochs, and results on OpenWebText is based on 200K steps.}
	\begin{tabular}{ccccccc}
		\toprule
		 &  & \multicolumn{2}{c}{\textbf{ImageNet (Acc. $\uparrow$)}} &  \multicolumn{3}{c}{\textbf{OpenWebText (Val Loss. $\downarrow$)}} \\
		Optimizer & Types of Model &  307M & 632M & 124M &  774M  & 1436M\\
		\midrule
            AdamW            & ViT/GPT2      & 81.7 & 80.8 & 2.867 & 2.492 & 2.435\\ 
            AdamW            & V-DNT/L-DNT   & \textbf{82.1} & \textbf{81.9} & \textbf{2.863} & \textbf{2.481} & \textbf{2.396}\\ \hline
            
		mSGDW             & ViT/GPT2      & 78.2 & 73.5 & 2.906 & 2.544 & 2.472\\
		
            mSGDW             & V-DNT/L-DNT   & \textbf{81.5} & \textbf{81.2} & \textbf{2.849} & \textbf{2.503} & \textbf{2.408}\\
		\bottomrule
	\end{tabular}
	\label{tab:experimental_results}
\end{table}
\subsection{Comparision of GPU memory used by mSGDW and AdamW}
We compare the memory usage by mSGDW and AdamW. The results are shown in Table~\ref{tab:memory_usage}.
\begin{table}[ht]
	\centering
	\caption{Comparision of GPU memory used by mSGDW and AdamW trained on 1.4B DNT model. DNT+AdamW means the network usage and the optimizer usage of GPU memory. $\dag$ denotes Theoretical calculated values, and $\ddag$ denotes practically  observed values.}
	\begin{tabular}{ccccc}
          \hline
                         &       AdamW   & mSGDW       & DNT+AdamW & DNT+mSGDW \\ 
		Memory       &      $11.5^{\dag}$ GB & $5.7^{\dag}$ GB & $\approx 67^{\ddag}$ GB & $\approx 61^{\ddag}$ GB \\
		
		\bottomrule
	\end{tabular}
	\label{tab:memory_usage}
\end{table}
\subsection{Ablation study}
We conduct ablation study of five different normalization methods. Figure~\ref{fig:five_settings} in the Appendix~\ref{appendix:five_network_setting} illustrates these five different network settings.
Let us brief introduce these five settings below:
\begin{itemize}[leftmargin=*]
\item Setting 1: Standard transformer with prenorm, which we abbreviate as S1; 
\item Setting 2: S1 + QKNorm; 
\item Setting 3: S2 + InputNorm; 
\item Setting 4: 2 PreNorms + MidNorm + QKNorm + InputNorm; 
\item Setting 5: only 1 PreNorm before self-attention + MidNorm + QKNorm + InputNorm. 
\end{itemize}
We use momentum mSGDW for all training in this subsection. All models were trained with the same hyper-parameters. The results are shown in Figure~\ref{exp:ablation_study}.

\begin{figure}[htbp]
	\centering
	\begin{minipage}{0.495\linewidth}
		\centering
		\includegraphics[width=0.99\linewidth,height=0.8\linewidth]{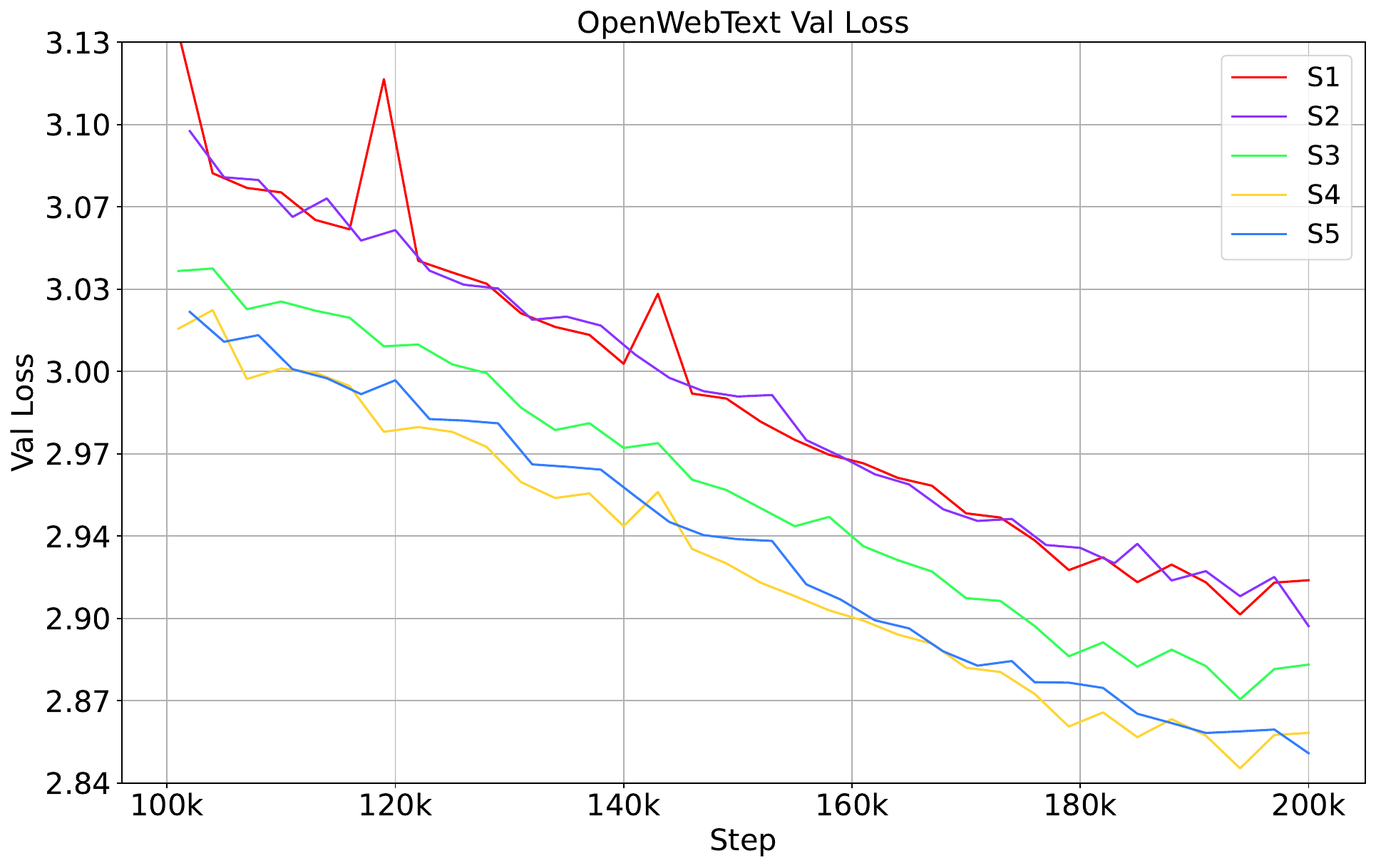}
		\label{chutian1}
	\end{minipage}
        \begin{minipage}{0.495\linewidth}
		\centering
		\includegraphics[width=0.99\linewidth,height=0.8\linewidth]{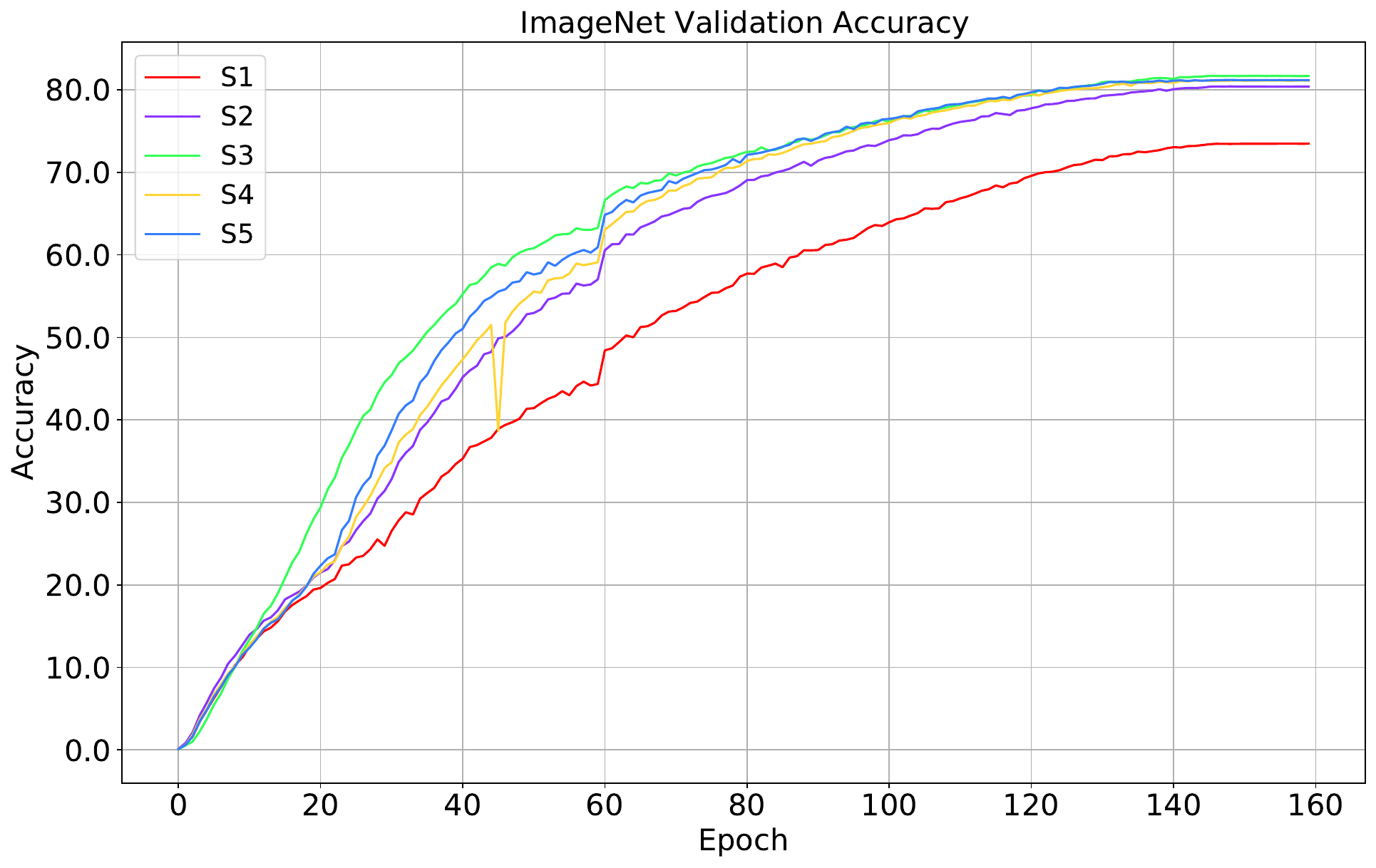}
		\label{chutian1}
	\end{minipage}
	\caption{Ablation study of different settings using  mSGD optimizer on ImageNet and OpenWebText. Left side shows accuracy curve of Huge vision model (632M) on ImageNet. Right side shows  the validation loss of language model (124M) on OpenWebText.}
\label{exp:ablation_study}
\end{figure}

We have the following observations,
\begin{itemize}[leftmargin=*]
	\item On the OpenWebText dataset, the original PreNorm setting (S1) shows the worst performance. The performance of S2 is similar to that of S1. S3 with input obtained a better performance. Finally, S4 and S5 obtained the best performance. Meanwhile, the performance of S4 and S5 is similar. 
	\item  On the ImageNet dataset, the original PreNorm setting (S1) is significantly worse than the other four Settings. S3 achieves the best setting, and S4 and S5 also obtain excellent performance.
\end{itemize}

We would like to point out that although in the main body, we recommend the Setttings 4 and 5, but we also observe that S3 sometimes achieves a good performance.

\section{Conclusion}
We have introduced a novel architecture, Deeply Normalized Transformer (DNT), which 
enables efficient 
training with vanilla momentum SGDW (mSGDW), achieving performance on par with AdamW-optimized Transformers. 
Unlike traditional approaches that rely on sophisticated optimizers to address the challenges of heavy-tailed gradient distributions, our DNT properly integrated normalization techniques into the architecture of Transformer to regulate the Jacobian matrices of each block, effectively balance the contributions of weights, activations, and their interactions, and thus make the gradient distribution concentrated. 
Our findings demonstrated that a properly designed architecture can make a simple optimizers like mSGDW just as effective as sophisticated ones 
and opened new opportunities for creating more efficient, scalable, and accessible Transformer models.



\bibliography{main.bib}
\bibliographystyle{iclr2025_conference}

\newpage

\appendix

\section{Proof of Proposition 2, 3 and 5}
Propostion 1 and Proposition 4 are very easy to prove, we have given a brief proof in the main body. Therefore, in the appendix, we only provide the proof of Proposition 2, 3 and 5.

\subsection{Proof of Proposition 2 on PreNorm}
\begin{proof}

A  single-head self-attention can be defined as
\begin{equation*}
    \bY = \bW_v \bX \bA,
\end{equation*}

\noindent where $\bP = \bX^{\top} {\bW_q}^{\top} {\bW_k} {\bX},  \quad \bA = \operatorname{softmax}(\frac{\bP}{\sqrt{d_q}})$.
$\bA$ is called as the attention matrix and $\frac{\bP}{\sqrt{d_q}}$ is called as the logit,  $\bA \in \mathcal{R}^{n\times n}, \bX \in \mathcal{R}^{d\times n}, \bW_v \in \mathcal{R}^{d_v\times d}$ . Here, our goal is to calculate $\frac{\partial \text{vec}(\bY)}{\partial \text{vec}(\bX)} $.

By vectorization of $ \bY = \bW_v \bX \bA$, we have 
\begin{equation*}
     \partial \text{vec}(\bY) =  (\bA^\top \otimes \bW_v) \partial \text{vec}(\bX)  + (\bI_n \otimes \bW_v\bX) \partial \text{vec}(\bA).
\end{equation*}

Bringing in all the terms, we get the following formula

{\footnotesize
\begin{equation*}
  \frac{\partial \text{vec}(\bY)}{\partial \text{vec}(\bX)} = (\bA^\top \otimes \bW_v) + (\bI_n \otimes \bW_v\bX)  \frac{\bJ}{\sqrt{d_q}}  \left( 
  (\bX^{\top}{\bW_k}^{\top}{\bW_q} \otimes \bI_n)\bC_{dn} + (\bI_n \otimes \bX^{\top}{\bW_q}^{\top}{\bW_k}) 
 \right),
\end{equation*}
}
\noindent where $\bC_{dn}$ is the communication matrix.

where $\frac{\partial \text{vec}(\bY)}{\partial \text{vec}(\bA)} =  \bI_n \otimes \bW_v\bX$, $\frac{\partial \text{vec}(\bA)}{\partial \text{vec}(\bP)} = \frac{\bJ}{\sqrt{d_q}}$, and we have
$$
\bJ  = \text{blockdiag}(\text{diag}(\bA_{:,1}) - \bA_{:,1} \bA_{:,1}^\top, \dots, \text{diag}(\bA_{:,n}) - \bA_{:,n} \bA_{:,n}^\top).
$$

$\bJ$ is a function of $\bA$, and $\bA$ is a function of $\bX$ associated with softmax function. Obviously, $\frac{\partial \text{vec}(\bY)}{\partial \text{vec}(\bX)}$ is a high-order function of $\bX$ and $\bJ$ makes the analysis more complex. 

Here, we give an analysis of the Jacobian matrix of the linear attention module, where $\bA = \frac{\bP}{\sqrt{d_q}}$ and $\bP = \bX^{\top} {\bW_q}^{\top} {\bW_k} {\bX}$. For the linear attention, we have the Jacobian matrix as,

{\footnotesize
\begin{equation}
  \frac{\partial \text{vec}(\bY)}{\partial \text{vec}(\bX)} = (\bA^\top \otimes \bW_v) +   \frac{(\bI_n \otimes \bW_v\bX)}{\sqrt{d_q}}  \left( 
  (\bX^{\top}{\bW_k}^{\top}{\bW_q} \otimes \bI_n)\bC + (\bI_n \otimes \bX^{\top}{\bW_q}^{\top}{\bW_k}) \label{eq:dp_dx}
 \right).
\end{equation}
}

Obviously, if the norm of each feature vector for each token is large, the magnitude of each element in $\frac{\partial \text{vec}(\bY)}{\partial \text{vec}(\bX)}$ will have large probability to be large, and the singular value of $\frac{\partial \text{vec}(\bY)}{\partial \text{vec}(\bX)}$ may be magnified second-orderly by the norm of each column in $\bX$.

if $\bX' = \bg \odot \bX$ with each column in $\bx'_i$ having a magnitude $\bg_i$, with the same $\bW_q, \bW_k, \bW_v$, we have $\bY = \bY'$ given $\bY = \operatorname{Self-Attention}(\bX)$ and $\bY'= \operatorname{Self-Attention}(\bX')$ because we will obtain the same input as the self-attention after the PreNorm.

After PreNorm, we have $\operatorname{PreNorm}(\bX) = \operatorname{PreNorm}(\bX')$, according to Equation~\ref{eq:jacobian_of_dY_dX}, we have $\frac{\partial \text{vec}(\bY)}{\partial \text{vec}(\bX)} = 	\frac{\partial \text{vec}(\bY')}{\partial \text{vec}(\bX')}$.

Until here, we have prove the Proposition~\ref{proposition_inputnorm}.


Furthermore, we would like to conduct a deeper analysis of the gradient of the loss with respect to the weights. In a backpropagation, since we have obtained $\frac{\partial \mathcal{L}}{\partial \text{vec}{(\bY)}}$, we would like to further analyze $\frac{\partial L}{\partial \text{vec}{(\bW_q)}}, \frac{\partial L}{\partial \text{vec}{(\bW_k)}}, \frac{\partial L}{\partial\text{vec}{(\bW_v)}}$.

For the weight matrix $\bW_q$,  we have
\begin{equation}
\begin{aligned}
	\frac{\partial \mathcal{L}}{\partial \text{vec}{(\bW_q)}} &=  \frac{\partial \mathcal{L}}{\partial \text{vec}{(\bY)}} \frac{\partial \text{vec}{(\bY)}}{\partial \text{vec}{(\bA)}}  \frac{\partial \text{vec}{(\bA)}}{\partial \text{vec}{(\bP)}}  \frac{\partial \text{vec}{(\bP)}}{\partial \text{vec}{(\bW_q)}}, \\
	&=  \frac{\partial \mathcal{L}}{\partial \text{vec}{(\bY)}}  \left(  \bI_n \otimes \bW_v\bX \right) \frac{\bJ}{\sqrt{d_q}} \left( (\bW_k\bX)^{\top} \otimes \bX^{\top} \right) \bC.
\end{aligned}
\label{eq:gradient_wq}
\end{equation}

For the weight matrix $\bW_k$,  we have
\begin{equation}
\begin{aligned}
	\frac{\partial \mathcal{L}}{\partial \text{vec}{(\bW_k)}} &=  \frac{\partial \mathcal{L}}{\partial \text{vec}{(\bY)}} \frac{\partial \text{vec}{(\bY)}}{\partial \text{vec}{(\bA)}}  \frac{\partial \text{vec}{(\bA)}}{\partial \text{vec}{(\bP)}}  \frac{\partial \text{vec}{(\bP)}}{\partial \text{vec}{(\bW_k)}}, \\
	&=  \frac{\partial \mathcal{L}}{\partial \text{vec}{(\bY)}}  \left(  \bI_n \otimes \bW_v\bX \right) \frac{\bJ}{\sqrt{d_q}} \left(\bX^{\top} \otimes (\bW_q \bX)^{\top} \right).
\end{aligned}
\label{eq:gradient_wk}
\end{equation}

For the weight matrix $\bW_v$, we know that $\text{vec}(\bY) =\left( (\bX \bA)^{\top} \otimes \bI \right) \text{vec}(\bW_v)$, thus we have
\begin{equation}
\begin{aligned}
	\frac{\partial \mathcal{L}}{\partial \text{vec}{(\bW_v)}} &=  \frac{\partial \mathcal{L}}{\partial \text{vec}{(\bY)}} \frac{\partial \text{vec}{(\bY)}}{\partial \text{vec}{(\bW_v)}}, \\
	&=  \frac{\partial \mathcal{L}}{\partial \text{vec}{(\bY)}} \left( (\bA^{\top} \bX^{\top}) \otimes \bI \right).
\end{aligned}
\label{eq:gradient_wv}
\end{equation}

We can see that in Equations~\ref{eq:gradient_wq}~\ref{eq:gradient_wk}~\ref{eq:gradient_wv}, the gradients of loss with respect to $\bW_q, \bW_k, \bW_v$ are all related to $\bX$. After PreNorm, $\bX$ is in a relatively stable range, thus, we can promise the range of value $\bX$ will not greatly affect the gradients of  $\bW_q, \bW_k, \bW_v$.

In conclusion, normalization of $\bX$ can help stablize the gradient of the loss function with respect to $\bW_q, \bW_k, \bW_v$, and meanwhile help make $\frac{\partial \text{vec}(\bY)}{\partial \text{vec}(\bX)}$ more stable and flat.
\end{proof}

\ 

\ 

\subsection{Proof of Proposition 3 on MidNorm}

\begin{proof}
Starting with the definition $\bW_1 =\frac{\bW}{\|\by\|}$ where $\by = \bW\bx$ and $\bW \in \mathcal{R}^{m \times n}$, let's derive the relationship between singular values:

For a random matrix $\bW$ with i.i.d. entries (mean 0, variance $\sigma^2_W$) and a random vector $\bx$ with i.i.d. entries (mean 0, variance $\sigma^2_x$):

$$\mathbb{E}[\|\by\|^2] = \mathbb{E}[\|\bW\bx\|^2] = \mathbb{E}[\bx^T \bW^T \bW \bx]$$

Using the trace property:
$$\mathbb{E}[\bx^T \bW^T \bW \bx] = \mathbb{E}[\text{tr}(\bx^T \bW^T \bW \bx)] = \mathbb{E}[\text{tr}(\bW \bx \bx^T \bW^T)]$$

With $\bx$ and $\bW$ independent, and $\mathbb{E}[\bx\bx^T] = \sigma_x^2 \bI_n$:
$$\mathbb{E}[\text{tr}(\bW \bx \bx^T \bW^T)] = \mathbb{E}[\text{tr}(\bW \sigma_x^2 \bI_n \bW^T)] = \sigma_x^2 \mathbb{E}[\text{tr}(\bW \bW^T)]$$

For $\bW$ with i.i.d. entries, $\mathbb{E}[\text{tr}(\bW \bW^T)] = m \cdot n \cdot \sigma_W^2$

Therefore:
$$\mathbb{E}[\|\by\|^2] = \sigma_x^2 \cdot m \cdot n \cdot \sigma_W^2 = m \cdot n \cdot \sigma_W^2 \cdot \sigma_x^2$$

By concentration of measure principles, $\|\by\|^2$ concentrates around its expectation with high probability:

$$\|\by\|^2 \approx \mathbb{E}[\|\by\|^2] = m \cdot n \cdot \sigma_W^2 \cdot \sigma_x^2$$

Taking the square root:
$$\|\by\| \approx \sqrt{m \cdot n} \cdot \sigma_W \cdot \sigma_x \text{ with high probability}$$

For the SVD of $\bW = \bU\Sigma \bV^T$, where $\Sigma$ contains singular values $\sigma_i(\bW)$, the singular values of $\bW_1$ are:
$$\sigma_i(\bW_1) = \sigma_i\left(\frac{\bW}{\|\by\|}\right) = \frac{\sigma_i(\bW)}{\|\by\|}$$

Substituting our concentration result:
$$\sigma_i(\bW_1) \approx \frac{\sigma_i(\bW)}{\sqrt{m \cdot n} \cdot \sigma_W \cdot \sigma_x}$$

For large random matrices with i.i.d. entries, random matrix theory~\citep{matrix_analysis_horn2012matrix, random_matrix_theory_tao2012topics} tells us that the largest singular value follows:
$$\sigma_1(\bW) \approx (\sqrt{m} + \sqrt{n}) \sigma_W$$

Substituting this into our expression:
\begin{equation}
\sigma_1(\bW_1) \approx \frac{(\sqrt{m} + \sqrt{n}) \sigma_W }{\sqrt{m \cdot n} \cdot \sigma_W \cdot \sigma_x} = \frac{\sqrt{m} + \sqrt{n}}{\sqrt{m \cdot n} \cdot \sigma_x}.
\label{eq:midnorm_deriviation_result}
\end{equation}

This derivation result in Equation~\ref{eq:midnorm_deriviation_result} shows that in high dimensions, the largest singular value of $\bW_1$ becomes essentially deterministic, depending only on the dimensions of $\bW$ and the statistical property $\sigma_x$ (standard variance of each entry in $\bx$) of the random vector $\bx$. 
If $m = n$, then we have, $\sigma_1(\bW_1) \approx \frac{2}{\sqrt{m} \cdot \sigma_x}$.
\end{proof}

\

\

\subsection{Proof of Proposition 5 on QKNorm}

\begin{proof}
Self-attention with QKNorm~\citep{scaling22b_dehghani2023scaling} is defined as:
\begin{equation*}
    \bY = \bW_v \bX \bA, 
\end{equation*}

\noindent where $\bA' = \operatorname{softmax}(\frac{\bP'}{\sqrt{d_h}})$, $\bP' = \bQ'^{\top} \bK'$, and $\bq'_i$ and $\bk'_i$ are the i-th column and the j-th column in $\bQ'$ and $\bK'$ individually, and we define
\begin{equation*}
\centering
\begin{aligned} 
\bq'_i &= \operatorname{RMSN}({\bW_q} \bx_i ) = \boldsymbol{\gamma}_q \odot   \frac{ \sqrt{d_h} {\bW_q} \bx_i }{ {\| {\bW_q} \bx_i  \|}_2 } =  \sqrt{d_h} \operatorname{diag}( \boldsymbol{\gamma} _q) \frac{ {\bW_q} \bx_i  }{ {\| {\bW_q} \bx_i  \|}_2} , \\
\bk'_j &= \operatorname{RMSN}({\bW_k} \bx_j ) = \boldsymbol{\gamma}_k \odot   \frac{ \sqrt{d_h} {\bW_k} \bx_j }{ {\| {\bW_k} \bx_j  \|}_2 } =  \sqrt{d_h} \operatorname{diag}( \boldsymbol{\gamma} _k) \frac{ {\bW_k} \bx_j  }{ {\| {\bW_k} \bx_j  \|}_2}.
\end{aligned}
\end{equation*}

To facilitate our derivation, we will use $\bQ = \bW_q \bX$ and $\bK = \bW_k \bX$ as before, we use $\bq_i$ and $\bk_j$ to denote the i-th column and the j-th column in $\bQ$ and $\bK$ individually. Thus, we can denote $\bq'_i = \operatorname{RMSN}(\bq_i)$ and $\bk'_j = \operatorname{RMSN}(\bk_j)$.

Therefore, according to the product rule and chain rule, we can denote the Jacobian matrix of $\bY$ with respect to $\bX$ as follows: 
\begin{scriptsize}
\begin{align*}
  \frac{\partial \text{vec}(\bY)}{\partial \text{vec}(\bX)} &= (\bA'^\top \otimes \bW_v) + (\bI_n \otimes \bW_v\bX) \frac{\partial \text{vec}(\bA')}{\partial \text{vec}(\bX)} \\
  &= (\bA'^\top \otimes \bW_v) + (\bI_n \otimes \bW_v\bX) \frac{\partial \text{vec}(\bA')}{\partial \text{vec}(\bP')} \frac{\partial \text{vec}(\bP')}{\partial \text{vec}(\bX)} \\
  &= (\bA'^\top \otimes \bW_v) + (\bI_n \otimes \bW_v\bX) \frac{\partial \text{vec}(\bA')}{\partial \text{vec}(\bP')} \left( \frac{\partial \text{vec}(\bP')}{\partial \text{vec}(\bQ')} \frac{\partial \text{vec}(\bQ')}{\partial \text{vec}(\bX')} + \frac{\partial \text{vec}(\bP')}{\partial \text{vec}(\bK')} \frac{\partial \text{vec}(\bK')}{\partial \text{vec}(\bX')} \right) \\
  &= (\bA'^\top \otimes \bW_v) + (\bI_n \otimes \bW_v\bX) \frac{\partial \text{vec}(\bA')}{\partial \text{vec}(\bP')} \left( \frac{\partial \text{vec}(\bP')}{\partial \text{vec}(\bQ')} \frac{\partial \text{vec}(\bQ')}{\partial \text{vec}(\bQ)}
  \frac{\partial \text{vec}(\bQ)}{\partial \text{vec}(\bX)} + \frac{\partial \text{vec}(\bP')}{\partial \text{vec}(\bK')} \frac{\partial \text{vec}(\bK')}{\partial \text{vec}(\bK)} \frac{\partial \text{vec}(\bK)}{\partial \text{vec}(\bX)} \right) \\
\end{align*}
\end{scriptsize}

To derive out $\frac{\partial \text{vec}(\bY)}{\partial \text{vec}(\bX)}$, we need to derive out each term in the above equation.

Since we know $\bP' = \bQ'^{\top} \bK'$, then we have,
\begin{equation*}
    \frac{\partial\text{vec}(\bP')}{\partial\text{vec}(\bK')} = \bI \otimes \bQ'^{\top}
\end{equation*}

Similarly, we have
\begin{equation*}
\frac{\partial\text{vec}(\bP')}{\partial\text{vec}(\bQ')} = (\bK'^{\top} \otimes \bI) \cdot \bC_{dN}
\end{equation*}
\noindent where $\bC_{dN}$ is the communication matrix.

We have, 
\begin{align*}
	\bJ_{\bQ}^{\bQ'}  &= \frac{\partial\text{vec}(\bQ')}{\partial\text{vec}(\bQ)} = \text{blockdiag}\left(\frac{\partial\bq'_1}{\partial\bq_1}, \frac{\partial\bq'_2}{\partial\bq_2}, \ldots, \frac{\partial\bq'_N}{\partial\bq_N}\right), \\
	\bJ_{\bK}^{\bK'}  &= \frac{\partial\text{vec}(\bK')}{\partial\text{vec}(\bK)} = \text{blockdiag}\left(\frac{\partial\bk'_1}{\partial\bk_1}, \frac{\partial\bk'_2}{\partial\bk_2}, \ldots, \frac{\partial\bk'_N}{\partial\bk_N}\right)
\end{align*}
where
\begin{equation*}
	\frac{\partial\bq'_i}{\partial\bq_i} = \frac{\sqrt{d_h}}{\|\bq_i\|} \operatorname{diag}( \boldsymbol{\gamma} _q)  \left(\bI - \frac{\bq_i\bq_i^{\top}}{\|\bq_i\|^2}\right).
\end{equation*}


We have,
\begin{align*}
	\frac{\partial\text{vec}(\bQ)}{\partial\text{vec}(\bX)} = \bI \otimes \bW_q, \\
	\frac{\partial\text{vec}(\bK)}{\partial\text{vec}(\bX)} =  \bI \otimes \bW_k.
\end{align*}

we also have
\begin{align*}
	\frac{\partial\bq'_i}{\partial\bx_i} &= \frac{\partial\bq'_i}{\partial\bq_i}  \frac{\partial\bq_i}{\partial\bx_i}  = \frac{\sqrt{d_h}}{\|\bq_i\|} \operatorname{diag}( \boldsymbol{\gamma} _q)  \left(\bI - \frac{\bq_i\bq_i^{\top}}{\|\bq_i\|^2}\right) \bW_q = \sqrt{d_h} \operatorname{diag}( \boldsymbol{\gamma} _q)  \left(\bI - \frac{\bq_i\bq_i^{\top}}{\|\bq_i\|^2}\right) \eqnmarkbox[magenta]{Psi2}{ \frac{\bW_q}{\|\bW_q \bx_i\|}} , \\
	\frac{\partial\bk'_j}{\partial\bx_j} &= \frac{\partial\bk'_j}{\partial\bk_j}  \frac{\partial\bk_j}{\partial\bx_j}  = \frac{\sqrt{d_h}}{\|\bk_j\|} \operatorname{diag}( \boldsymbol{\gamma} _k)  \left(\bI - \frac{\bk_j\bk_j^{\top}}{\|\bk_j\|^2}\right) \bW_k = \sqrt{d_h} \operatorname{diag}( \boldsymbol{\gamma} _k)  \left(\bI - \frac{\bk_j\bk_j^{\top}}{\|\bk_j\|^2}\right) \eqnmarkbox[magenta]{Psi2}{\frac{\bW_k}{\|\bW_k \bx_i\|}}.
\end{align*}

In Proposition 3, we have proved that in a high-dimensional setting, the singular values of $\frac{\bW}{\|\bW \bx \|}$ is independent to the magnitude of $\bW$. Thus, until now, we have proved the Proposition 5.
\end{proof}

Further, we would like to discuss the Jacobian matrix $ \frac{\partial \text{vec}(\bY)}{\partial \text{vec}(\bX)} $ after QKNorm. We have,
{\footnotesize
\begin{equation}
  \frac{\partial \text{vec}(\bY)}{\partial \text{vec}(\bX)} = (\bA'^\top \otimes \bW_v) + (\bI_n \otimes \bW_v\bX)  \frac{\bJ}{\sqrt{d_h}}  \left(   
   (\bK'^{\top} \otimes \bI) \cdot \bC_{dN}  \bJ_{\bQ}^{\bQ'}    (\bI \otimes \bW_q) +  (\bI \otimes \bQ'^{\top})  \bJ_{\bK}^{\bK'}   (\bI \otimes \bW_k)
 \right).
 \label{eq:QK_jacobian}
\end{equation}
}

It should be noted that 
\begin{itemize}[leftmargin=*]
	\item $\bK'$ and $\bQ'$ are two normalized terms that have relatively stable range of values.
	\item the Jacobian matrix of $ \bJ_{\bQ}^{\bQ'}$ is relatively independent to the magnitude of $\bQ$ and $ \bJ_{\bK}^{\bK'}$ is relatively independent to the magnitude of $\bK$ in a high-dimensional setting.
	\item QKNorm cannot fully replace the value of PreNorm because $ \frac{\partial \text{vec}(\bY)}{\partial \text{vec}(\bX)} $ in Equation~\ref{eq:QK_jacobian} is directly affected by $\bX$.
\end{itemize}

QKNorm will elliviate the influence of the magnitude of $\bW_q$ and $\bW_k$ on the $\frac{\partial \text{vec}(\bY)}{\partial \text{vec}(\bX)}$. In the traditional self-attention, $\frac{\partial \text{vec}(\bY)}{\partial \text{vec}(\bX)}$ is largely affected by $\bW_q^{\top} \bW_k$. However, after QKNorm, $\frac{\partial \text{vec}(\bY)}{\partial \text{vec}(\bX)}$ will only be affected by independent $\bW_q$ or $\bW_k$ instead of the joint term $\bW_q^{\top} \bW_k$. This is important for the training stability because the singular values of $\bW_q^{\top} \bW_k$ will increase extremely fast when both singular values of $\bW_q$ and $\bW_k$ are increasing.

\ 

\

\section{Experimental details}
\begin{table}[H]
	\centering
	\small  
	\caption{Model configurations, peak learning rate and weight decay for different optimizers.}
	\label{tab:model_configs}
	\begin{tabular}{lcccccccc}
		\toprule
		\multirow{2}{*}{Acronym} & \multirow{2}{*}{Size} & \multirow{2}{*}{d\_model} & \multirow{2}{*}{n\_head} & \multirow{2}{*}{depth} & \multicolumn{2}{c}{AdamW} & \multicolumn{2}{c}{mSGDW} \\
		\cmidrule(lr){6-7} \cmidrule(lr){8-9}
		&  &  &  &  & LR & WD & LR & WD \\
		\midrule
		L-DNT-Small & 124M & 768 & 12 & 12 & 6e-4 & 0.1    & 1.0 & 1e-4\\
		L-DNT-Large & 774M & 1280 & 20 & 36 & 6e-4 & 0.1 & 1.0 & 1e-4\\
            L-DNT-XL    & 1436M & 1536 & 24 & 48 & 6e-4 & 0.1 & 1.0 & 1e-4\\
		V-DNT-Large & 307M & 1024 & 16 & 24 & 1e-3 & 0.1 & 0.5 & 2e-4\\
		V-DNT-Huge &  632M & 1280 & 16 & 32 & 1e-3 & 0.1 & 0.1 & 1e-3\\
		\bottomrule
	\end{tabular}
\label{table:model_config}
\end{table}

We conducted experiments on two popular architectures: Vision Transformer (ViT) and Generative Pretrained Transformer (GPT). Our implementation leverages established repositories: timm~\citep{timm_rw2019timm} for ViT and nanoGPT~\citep{nanogpt_Karpathy2022} for GPT models.  We utilized five model configurations: L-DNT-Small (124M parameters), L-DNT-Large (774M parameters), L-DNT-XL (1436M parameters), V-DNT-Large (307M parameters), and V-DNT-Huge (632M parameters). Model specifications including hidden dimension (d\_model), number of attention heads (n\_head), and network depth are detailed in Table~\ref{table:model_config}. Training was conducted using PyTorch~\citep{pytorch_paszke2019pytorch} with bfloat16 precision GPUs, employing a cosine learning rate schedule.

All language models were trained on OpenWebText,  using GPT-2 tokenizer. The training dataset contains 9B tokens, with a validation set of 4.4M tokens, following the train-validation split from nanoGPT.
We employed distributed data parallel training with gradient accumulation. All models were trained using bfloat16 precision. The 124M models were trained on machines with 8 GPUs, 774M models were trained with 16 A800 GPUs, while 1436M models were trained with 32 GPUs.  Our global batch sizes for 125M, 770M and 1436M models are 480, 512 and 512 individually. In Sophia~\citep{sophia_liu2023sophia}, they use 480 global batch size for all models.
For all language models, we used 2000 steps or learning rate warmup to the maximum learning rate, and then used a cosine learning rate decay. It takes around four days to train 200K steps for the 1.4B model on 32 GPUs.

All vision models were trained on ImageNet dataset. We trained each 150 epoches as~\citep{adan_xie2024adan}. We used a learning rate warmup of 60 epochs to the maximum learning rate, and then used a cosine learning rate decay.

For our experiments, we focused on comparing AdamW and mSGDW optimizers. The hyperparameters for AdamW were carefully tuned, with $\beta_1 = 0.9$ and $\beta_2 = 0.95$, following the dominant configuration in LLM pre-training literature. For weight decay, we used 0.1 for AdamW as~\citep{nanogpt_Karpathy2022, sophia_liu2023sophia}. We used the recommended learning rate by nanoGPT for AdamW in GPT. Since our DNT is robust to large learning rate, we use 6e-4 for all our L-DNT models and 1e-3 for all our V-DNT models for AdamW.
For mSGDW, we simply use a rough grid search for the learning rate, we cannot search a fine-grained learning rate and weight decay due to its requiring a lot of resources. For the momentum in mSGDW, we used a default 0.9 for all experiments. We use the implementation~\footnote{\url{https://github.com/huggingface/pytorch-image-models/blob/main/timm/optim/sgdw.py}} of mSGDW from timm~\citep{timm_rw2019timm}. This implementation is a decoupled weight decay regularization used in AdamW~\citep{adamw_IlyaLoshchilov2018FixingWD}. Note that mSGDW is not directly to add a weight decay in the original implementation of mSGD~\footnote{\url{https://pytorch.org/docs/stable/generated/torch.optim.SGD.html}} in the official PyTorch, it will have performance problem.

\begin{table}[H]
\small
	\centering
	\caption{Training configurations for ViT and V-DNT.}
		\begin{tabular}{l|cccc} 
			\hline
			training config & ViT-L/H ($224^{2}$)  & ViT-L/H ($224^{2}$)  & V-DNT-L/H   ($224^{2}$) & V-DNT-L/H   ($224^{2}$)\\ 
			\hline
			optimizer & AdamW  & mSGDW & AdamW & mSGDW\\  
			learning rate schedule & \multicolumn{4}{c}{cosine decay} \\
			peak learning rate & 1e-3 & 0.5/0.1 & 1e-3 & 0.5/0.1\\ 
			minimum learning rate & 1e-8 & 1e-8 & 1e-8 & 1e-8\\ 
			weight decay & 0.1 & 2e-4/1e-3  & 0.1 & 2e-4/1e-3 \\ 
			optimizer momentum & \(\beta_{1}, \beta_{2}=0.9,0.99\)  & \(\mu = 0.9\) & \(\beta_{1}, \beta_{2}=0.9,0.99\) & \(\mu = 0.9\)\\ 
			warmup epoches & 60 & 60 & 60 & 60 \\ 
			\hline
			weight init & \multicolumn{4}{c}{Truncated Xavier}\\ 
			batch size & \multicolumn{4}{c}{1024} \\ 
			training epochs & \multicolumn{4}{c}{150}  \\ 
			randaugment  & \multicolumn{4}{c}{\((9,0.5)\)} \\ 
			mixup  & \multicolumn{4}{c}{0.8} \\ 
			cutmix & \multicolumn{4}{c}{1.0} \\
			random erasing  & \multicolumn{4}{c}{0} \\ 
			label smoothing  & \multicolumn{4}{c}{0.1} \\ 
			stochastic depth  & \multicolumn{4}{c}{\(0.1 / 0.5\)} \\ 
			gradient clip & \multicolumn{4}{c}{None} \\ 
			exp. mov. avg. (EMA)  & \multicolumn{4}{c}{no}\\ 
			\hline
		\end{tabular}
	\label{tab:training_configuration}
\end{table}

\begin{table}[H]
\small
	\centering
	\caption{Training configurations for GPT and L-DNT.}
	\begin{tabular}{l|cccc} 
		\hline
		training config & GPT2-S/L/XL  & GPT2-S/L/XL & L-DNT-S/L/XL & L-DNT-S/L/XL \\ 
		\hline
		optimizer & AdamW  & mSGDW & AdamW & mSGDW\\  
		learning rate schedule & \multicolumn{4}{c}{cosine decay} \\
		peak learning rate & 6e-4/2.5e-4/1.5e-4 & 1.0 & 6e-4 & 1.0\\ 
		minimum learning rate & 6e-5 & 6e-5 & 6e-5 & 6e-5\\ 
		weight decay & 0.1 & 1e-4  & 0.1 & 1e-4 \\ 
		optimizer momentum & \(\beta_{1}, \beta_{2}=0.9,0.95\)  & \(\mu = 0.9\) & \(\beta_{1}, \beta_{2}=0.9,0.95\) & \(\mu = 0.9\)\\ 
		warmup steps & 2000 & 0 & 2000 & 0 \\ 
		\hline
		weight init &  \multicolumn{4}{c}{Xavier}  \\ 
		tokens seen each update &  \multicolumn{4}{c}{480K/512K/512K} \\
		max iters &  \multicolumn{4}{c}{200K} \\
		batch size &  \multicolumn{4}{c}{480/512/512} \\ 
		sequence length &  \multicolumn{4}{c}{1024} \\ 
		dropout &  \multicolumn{4}{c}{0.0} \\
		bfloat16 &  \multicolumn{4}{c}{True}  \\
		gradient clipping &  \multicolumn{4}{c}{1.0} \\
		\hline
	\end{tabular}     
	\label{tab:training_configuration}
\end{table}

\ 

\

\newpage

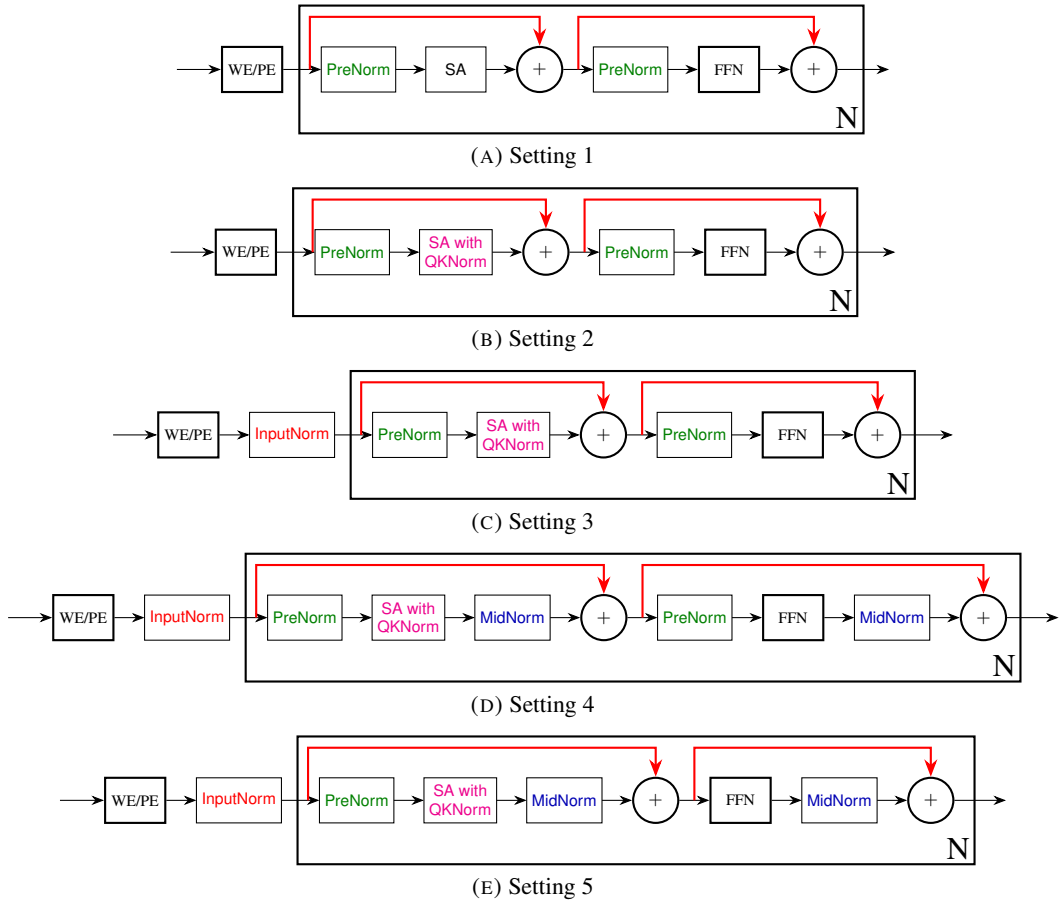
\begin{figure}[H]
	\centering
	\tiny

	\caption{Five different settings of normalizations evaluated in the ablation study.}

	\label{fig:five_settings}

\begin{tikzpicture}[
	box/.style={draw, minimum width=0.8cm, minimum height=0.6cm, thick},
	inputnorm/.style={draw, minimum width=0.8cm, minimum height=0.6cm, 
		text=red, font=\sffamily, align=center,   
		outer sep=0pt},
	prenorm/.style={draw, 
		minimum width=0.8cm, minimum height=0.6cm, 
		text=darkgreen, font=\sffamily, align=center,   
		outer sep=0.0pt},
	prenorm2/.style={draw, 
		minimum width=0.8cm, minimum height=0.6cm, 
		text=darkgreen, font=\sffamily, align=center,   
		outer sep=0.0pt},
	prenorm3/.style={draw,  rounded corners=3mm,
		minimum width=0.8cm, minimum height=0.6cm, 
		text=darkgreen, font=\sffamily, align=center,   
		outer sep=0.0pt},    
	midnorm/.style={draw, 
		minimum width=0.8cm, minimum height=0.6cm, 
		text=blue, font=\sffamily, align=center,   
		outer sep=0pt},
	qknorm/.style={draw, 
		minimum width=0.8cm, minimum height=0.6cm, 
		text=black, font=\sffamily, align=center,   
		outer sep=0pt},
	bigbox/.style={draw, minimum width=7.5cm, minimum height=1.7cm, thick},
	circ/.style={draw, circle, minimum size=0.6cm, thick},
	>=Stealth
	]

	\node[box] (we) at (0,0) {WE/PE};
	
	\node[bigbox, right=0.2cm of we, anchor=west] (enclosure) {};
	\node[anchor=south east] at (enclosure.south east) {\large N};
	
	\node[prenorm] (n2) at ($(enclosure.west) + (0.8,-0.0)$) {PreNorm};
	\node[qknorm, right=0.4cm of n2] (sa) {SA};
	\node[circ, right=0.4cm of sa] (plus1) {$+$};
	
	\node[prenorm, right=0.4cm of plus1] (n4) {PreNorm};
	\node[box, right=0.4cm of n4] (ffn) {FFN};
	\node[circ, right=0.4cm of ffn] (plus2) {$+$};
	

	\draw[->] ++(-1.0,0)  -- (we);
	\draw[->] (we) --  (n2);

	\draw[->] (n2) -- (sa);
	\draw[->] (sa) -- (plus1);
	\draw[->] (plus1) -- (n4);
	\draw[->] (n4) -- (ffn);
	\draw[->] (ffn) -- (plus2);
	\draw[->] (plus2) -- ++(1.0,0);
	
	\path (we) -- (n2) coordinate[pos=0.7] (midpoint);
	\draw[thick, red, <-] 
	(plus1) -- ++(0,0.7) -| (midpoint);
	
	
	\path (plus1) -- (n4) coordinate[pos=0.5] (midpoint);
	\draw[thick, red, <-] 
	(plus2) -- ++(0,0.7) -| (midpoint);
\end{tikzpicture}
\subcaption{Setting 1}

\begin{tikzpicture}[
	box/.style={draw, minimum width=0.8cm, minimum height=0.6cm, thick},
	inputnorm/.style={draw, minimum width=0.8cm, minimum height=0.6cm, 
		text=red, font=\sffamily, align=center,   
		outer sep=0pt},
	prenorm/.style={draw, 
		minimum width=0.8cm, minimum height=0.6cm, 
		text=darkgreen, font=\sffamily, align=center,   
		outer sep=0.0pt},
	prenorm2/.style={draw, 
		minimum width=0.8cm, minimum height=0.6cm, 
		text=darkgreen, font=\sffamily, align=center,   
		outer sep=0.0pt},
	prenorm3/.style={draw,  rounded corners=3mm,
		minimum width=0.8cm, minimum height=0.6cm, 
		text=darkgreen, font=\sffamily, align=center,   
		outer sep=0.0pt},    
	midnorm/.style={draw, 
		minimum width=0.8cm, minimum height=0.6cm, 
		text=blue, font=\sffamily, align=center,   
		outer sep=0pt},
	qknorm/.style={draw, 
		minimum width=0.8cm, minimum height=0.6cm, 
		text=magenta, font=\sffamily, align=center,   
		outer sep=0pt},
	bigbox/.style={draw, minimum width=7.5cm, minimum height=1.7cm, thick},
	circ/.style={draw, circle, minimum size=0.6cm, thick},
	>=Stealth
	]

	\node[box] (we) at (0,0) {WE/PE};
	
	\node[bigbox, right=0.2cm of we, anchor=west] (enclosure) {};
	\node[anchor=south east] at (enclosure.south east) {\large N};
	
	\node[prenorm] (n2) at ($(enclosure.west) + (0.8,-0.0)$) {PreNorm};
	\node[qknorm, right=0.4cm of n2] (sa) {SA with\\QKNorm};
	\node[circ, right=0.4cm of sa] (plus1) {$+$};
	
	\node[prenorm, right=0.4cm of plus1] (n4) {PreNorm};
	\node[box, right=0.4cm of n4] (ffn) {FFN};
	\node[circ, right=0.4cm of ffn] (plus2) {$+$};
	

	\draw[->] ++(-1.0,0)  -- (we);
	\draw[->] (we) --  (n2);

	\draw[->] (n2) -- (sa);
	\draw[->] (sa) -- (plus1);
	\draw[->] (plus1) -- (n4);
	\draw[->] (n4) -- (ffn);
	\draw[->] (ffn) -- (plus2);
	\draw[->] (plus2) -- ++(1.0,0);
	
	\path (n1) -- (n2) coordinate[pos=0.7] (midpoint);
	\draw[thick, red, <-] 
	(plus1) -- ++(0,0.7) -| (midpoint);
	
	
	\path (plus1) -- (n4) coordinate[pos=0.5] (midpoint);
	\draw[thick, red, <-] 
	(plus2) -- ++(0,0.7) -| (midpoint);
\end{tikzpicture}
\subcaption{Setting 2}

\begin{tikzpicture}[
	box/.style={draw, minimum width=0.8cm, minimum height=0.6cm, thick},
	inputnorm/.style={draw, minimum width=0.8cm, minimum height=0.6cm, 
		text=red, font=\sffamily, align=center,   
		outer sep=0pt},
	prenorm/.style={draw, 
		minimum width=0.8cm, minimum height=0.6cm, 
		text=darkgreen, font=\sffamily, align=center,   
		outer sep=0.0pt},
	prenorm2/.style={draw, 
		minimum width=0.8cm, minimum height=0.6cm, 
		text=darkgreen, font=\sffamily, align=center,   
		outer sep=0.0pt},
	prenorm3/.style={draw,  rounded corners=3mm,
		minimum width=0.8cm, minimum height=0.6cm, 
		text=darkgreen, font=\sffamily, align=center,   
		outer sep=0.0pt},    
	midnorm/.style={draw, 
		minimum width=0.8cm, minimum height=0.6cm, 
		text=blue, font=\sffamily, align=center,   
		outer sep=0pt},
	qknorm/.style={draw, 
		minimum width=0.8cm, minimum height=0.6cm, 
		text=magenta, font=\sffamily, align=center,   
		outer sep=0pt},
	bigbox/.style={draw, minimum width=7.5cm, minimum height=1.7cm, thick},
	circ/.style={draw, circle, minimum size=0.6cm, thick},
	>=Stealth
	]

	\node[box] (we) at (0,0) {WE/PE};
	\node[inputnorm, right=0.4cm of we] (n1) {InputNorm};
	
	\node[bigbox, right=0.2cm of n1, anchor=west] (enclosure) {};
	\node[anchor=south east] at (enclosure.south east) {\large N};
	
	\node[prenorm] (n2) at ($(enclosure.west) + (0.8,-0.0)$) {PreNorm};
	\node[qknorm, right=0.4cm of n2] (sa) {SA with\\QKNorm};
	\node[circ, right=0.4cm of sa] (plus1) {$+$};
	
	\node[prenorm, right=0.4cm of plus1] (n4) {PreNorm};
	\node[box, right=0.4cm of n4] (ffn) {FFN};
	\node[circ, right=0.4cm of ffn] (plus2) {$+$};
	
	\draw[->] (n1) --  (n2);

	\draw[->] ++(-1.0,0)  -- (we);
	\draw[->] (we) -- (n1);

	\draw[->] (n2) -- (sa);
	\draw[->] (sa) -- (plus1);
	\draw[->] (plus1) -- (n4);
	\draw[->] (n4) -- (ffn);
	\draw[->] (ffn) -- (plus2);
	\draw[->] (plus2) -- ++(1.0,0);
	
	\path (n1) -- (n2) coordinate[pos=0.7] (midpoint);
	\draw[thick, red, <-] 
	(plus1) -- ++(0,0.7) -| (midpoint);
	
	
	\path (plus1) -- (n4) coordinate[pos=0.5] (midpoint);
	\draw[thick, red, <-] 
	(plus2) -- ++(0,0.7) -| (midpoint);
\end{tikzpicture}
\subcaption{Setting 3}

	\begin{tikzpicture}[
	box/.style={draw, minimum width=0.8cm, minimum height=0.6cm, thick},
	inputnorm/.style={draw, minimum width=0.8cm, minimum height=0.6cm, 
		text=red, font=\sffamily, align=center,   
		outer sep=0pt},
	prenorm/.style={draw, 
		minimum width=0.8cm, minimum height=0.6cm, 
		text=darkgreen, font=\sffamily, align=center,   
		outer sep=0.0pt},
	prenorm2/.style={draw, 
		minimum width=0.8cm, minimum height=0.6cm, 
		text=darkgreen, font=\sffamily, align=center,   
		outer sep=0.0pt},
	prenorm3/.style={draw,  rounded corners=3mm,
		minimum width=0.8cm, minimum height=0.6cm, 
		text=darkgreen, font=\sffamily, align=center,   
		outer sep=0.0pt},    
	midnorm/.style={draw, 
		minimum width=0.8cm, minimum height=0.6cm, 
		text=blue, font=\sffamily, align=center,   
		outer sep=0pt},
	qknorm/.style={draw, 
		minimum width=0.8cm, minimum height=0.6cm, 
		text=magenta, font=\sffamily, align=center,   
		outer sep=0pt},
	bigbox/.style={draw, minimum width=10.3cm, minimum height=1.7cm, thick},
	circ/.style={draw, circle, minimum size=0.6cm, thick},
	>=Stealth
	]

	\node[box] (we) at (0,0) {WE/PE};
	\node[inputnorm, right=0.4cm of we] (n1) {InputNorm};
	
	\node[bigbox, right=0.2cm of n1, anchor=west] (enclosure) {};
	\node[anchor=south east] at (enclosure.south east) {\large N};
	
	\node[prenorm] (n2) at ($(enclosure.west) + (0.8,-0.0)$) {PreNorm};
	\node[qknorm, right=0.4cm of n2] (sa) {SA with\\QKNorm};
	\node[midnorm, right=0.4cm of sa] (n3) {MidNorm};
	\node[circ, right=0.4cm of n3] (plus1) {$+$};
	
	\node[prenorm, right=0.4cm of plus1] (n4) {PreNorm};
	\node[box, right=0.4cm of n4] (ffn) {FFN};
	\node[midnorm, right=0.4cm of ffn] (n5) {MidNorm};
	\node[circ, right=0.4cm of n5] (plus2) {$+$};
	
	\draw[->] (n1) --  (n2);

	\draw[->] ++(-1.0,0)  -- (we);
	\draw[->] (we) -- (n1);

	\draw[->] (n2) -- (sa);
	\draw[->] (sa) -- (n3);
	\draw[->] (n3) -- (plus1);
	\draw[->] (plus1) -- (n4);
	\draw[->] (n4) -- (ffn);
	\draw[->] (ffn) -- (n5);
	\draw[->] (n5) -- (plus2);
	\draw[->] (plus2) -- ++(1.0,0);
	
	\path (n1) -- (n2) coordinate[pos=0.7] (midpoint);
	\draw[thick, red, <-] 
	(plus1) -- ++(0,0.7) -| (midpoint);
	
	
	\path (plus1) -- (n4) coordinate[pos=0.5] (midpoint);
	\draw[thick, red, <-] 
	(plus2) -- ++(0,0.7) -| (midpoint);
\end{tikzpicture}
\subcaption{Setting 4}

   	\begin{tikzpicture}[
		box/.style={draw, minimum width=0.8cm, minimum height=0.6cm, thick},
		inputnorm/.style={draw, minimum width=0.8cm, minimum height=0.6cm, 
			text=red, font=\sffamily, align=center,   
			outer sep=0pt},
		prenorm/.style={draw, 
			minimum width=0.8cm, minimum height=0.6cm, 
			text=darkgreen, font=\sffamily, align=center,   
			outer sep=0.0pt},
		prenorm2/.style={draw, 
			minimum width=0.8cm, minimum height=0.6cm, 
			text=darkgreen, font=\sffamily, align=center,   
			outer sep=0.0pt},
		prenorm3/.style={draw,  rounded corners=3mm,
			minimum width=0.8cm, minimum height=0.6cm, 
			text=darkgreen, font=\sffamily, align=center,   
			outer sep=0.0pt},    
		midnorm/.style={draw, 
			minimum width=0.8cm, minimum height=0.6cm, 
			text=blue, font=\sffamily, align=center,   
			outer sep=0pt},
		qknorm/.style={draw, 
			minimum width=0.8cm, minimum height=0.6cm, 
			text=magenta, font=\sffamily, align=center,   
			outer sep=0pt},
		bigbox/.style={draw, minimum width=9.0cm, minimum height=1.7cm, thick},
		circ/.style={draw, circle, minimum size=0.6cm, thick},
		>=Stealth
		]

		\node[box] (we) at (0,0) {WE/PE};
		\node[inputnorm, right=0.4cm of we] (n1) {InputNorm};
		
		\node[bigbox, right=0.2cm of n1, anchor=west] (enclosure) {};
		\node[anchor=south east] at (enclosure.south east) {\large N};
		
		\node[prenorm] (n2) at ($(enclosure.west) + (0.8,-0.0)$) {PreNorm};
		\node[qknorm, right=0.4cm of n2] (sa) {SA with\\QKNorm};
		\node[midnorm, right=0.4cm of sa] (n3) {MidNorm};
		\node[circ, right=0.4cm of n3] (plus1) {$+$};
		
		\node[box, right=0.4cm of plus1] (ffn) {FFN};
		\node[midnorm, right=0.4cm of ffn] (n5) {MidNorm};
		\node[circ, right=0.4cm of n5] (plus2) {$+$};
		
		\draw[->] (n1) --  (n2);

		\draw[->] ++(-1.0,0)  -- (we);
		\draw[->] (we) -- (n1);

		\draw[->] (n2) -- (sa);
		\draw[->] (sa) -- (n3);
		\draw[->] (n3) -- (plus1);
		\draw[->] (plus1) -- (ffn);
		\draw[->] (ffn) -- (n5);
		\draw[->] (n5) -- (plus2);
		\draw[->] (plus2) -- ++(1.0,0);
		
		\path (n1) -- (n2) coordinate[pos=0.7] (midpoint);
		\draw[thick, red, <-] 
		(plus1) -- ++(0,0.7) -| (midpoint);
		
		
		\path (plus1) -- (n4) coordinate[pos=0.5] (midpoint);
		\draw[thick, red, <-] 
		(plus2) -- ++(0,0.7) -| (midpoint);
		
	\end{tikzpicture}
\subcaption{Setting 5}

\end{figure}

\section{Five different network settings}
\label{appendix:five_network_setting}
Figure~\ref{fig:five_settings} illustrates five different network settings. We have conducted a ablation study for these five settings in the main body part.

\ 

\

\section{Comparision of mSGDW and AdamW on larger models}
We further compare larger V-DNT and L-DNT models and original ViT and GPT2 models on ImageNet and OpenWebText using mSGDW and AdamW. The results are shown in Figures~\ref{fig:lheat_large_200k}, ~\ref{fig:xlheat_large_200k} and~\ref{fig:vheat_huge}.

We see that on OpenWebText, L-DNT-large with mSGDW achieves a comparable performance with L-DNT-large with AdamW and achieves a much better performance than GPT2-large with mSGDW. Meanwhile,  we find out that V-DNT-huge with mSGDW achieves a comparable performance with V-DNT-huge with AdamW and obtains a much better performance than ViT-huge with mSGDW.


\begin{figure}[H]
	\centering
	\includegraphics[width=0.8\textwidth,height=0.60\linewidth]{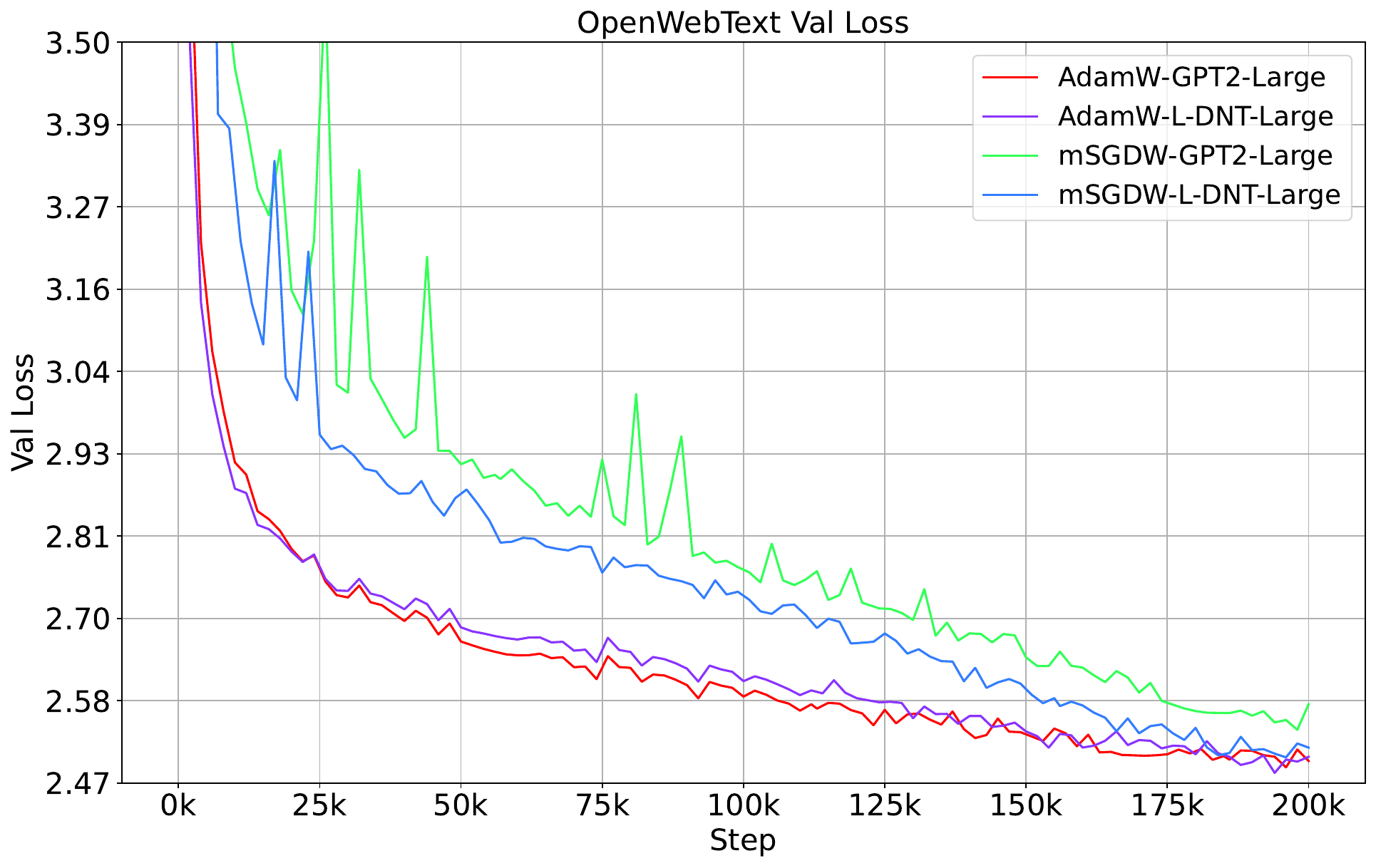}
	\caption{L-DNT-Large (774M) training with 200K in total on OpenWebText.}
\label{fig:lheat_large_200k}
\end{figure}

\begin{figure}[H]
	\centering
	\includegraphics[width=0.8\textwidth,height=0.60\linewidth]{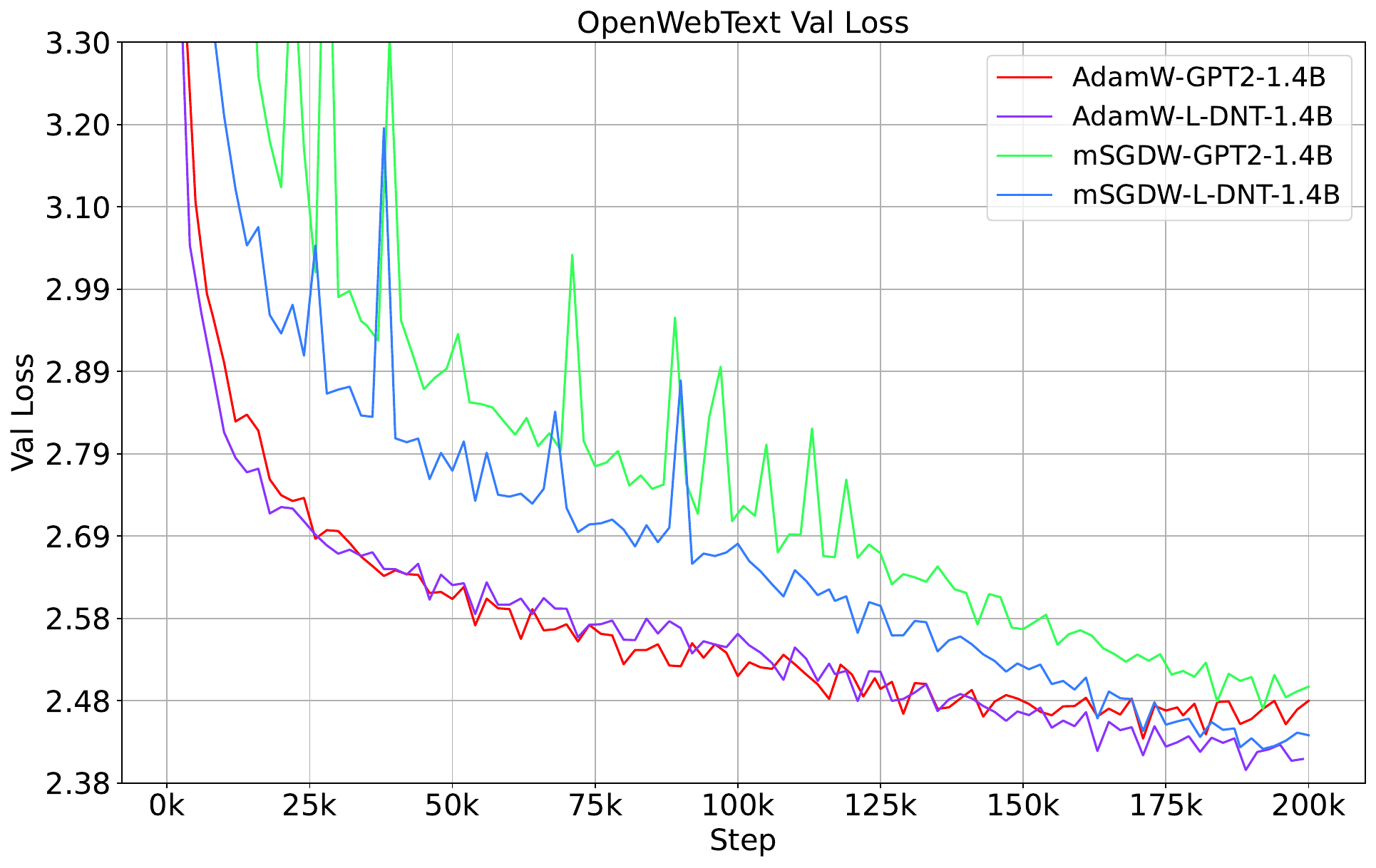}
	\caption{L-DNT-XL (1436M) training with 200K in total on OpenWebText.}
\label{fig:xlheat_large_200k}
\end{figure}

\begin{figure}[H]
	\centering
	\includegraphics[width=0.8\textwidth,height=0.60\linewidth]{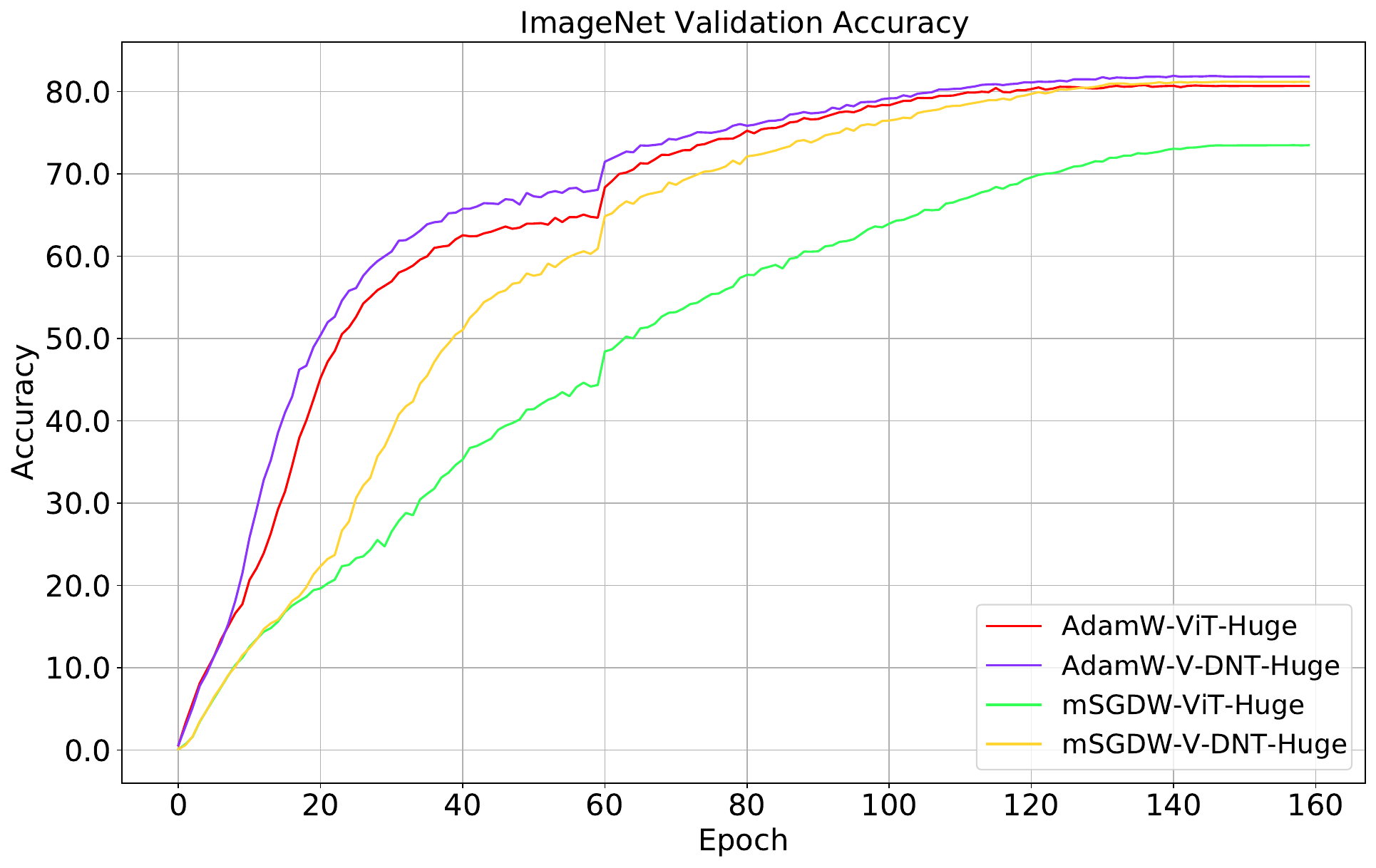}
	\caption{V-DNT-Huge (632M) on ImageNet}
\label{fig:vheat_huge}
\end{figure}

%




		

\

\

\section{A brief introduction to some existing optimizers}

Optimization methods in deep learning can be broadly categorized into first-order methods and second-order methods, each with distinct characteristics and applications. First-order optimization algorithms dominate deep learning due to their computational efficiency, particularly for high-dimensional and large-scale problems.
First-order methods rely primarily on gradient information to find the minimum or maximum of a function. Based on learning rate selection strategies, these methods can be divided into optimizers with fixed step size and optimizers with adaptive learning rate.

Stochastic Gradient Descent (SGD)~\citep{sgd_robbins1951stochastic} serves as the foundational algorithm for neural network optimization. It updates parameters in the opposite direction of the gradient of the objective function. While simple and effective, vanilla SGD can struggle with navigating ravines and saddle points in the loss landscape.
Momentum SGD (mSGD)~\citep{nesterov1983method} addresses the limitations of vanilla SGD by accelerating gradient descent in relevant directions while dampening oscillations. This method augments the gradient direction with a fraction of the update vector from the previous step, allowing faster convergence and helping escape local minima.
Other notable variants include signSGD~\citep{signsgd_bernstein2018signsgd}, which uses only the sign of gradients for updates; SVRG~\citep{svrg_johnson2013accelerating}, which reduces variance in stochastic gradients; LARS~\citep{lars_you2017large}, which adjusts learning rates layer-wise.

Adaptive methods revolutionized gradient-based optimization by incorporating two key innovations. First, they implement parameter-specific learning rate adaptation, performing smaller updates for frequently occurring features and larger updates for infrequent features. Second, they incorporate historical gradient information, often approximating second-order properties of the loss landscape.
AdaGrad~\citep{adagrad_duchi2011adaptive} adapts learning rates based on historical gradient information and is particularly effective for sparse data. RMSprop~\citep{rmsprop_hinton2012rmsprop} addresses AdaGrad's radically diminishing learning rates by using an exponentially weighted moving average. Adam~\citep{adam_kingma2014adam} combines momentum with adaptive learning rates, incorporating both first and second moments of gradients. AdamW~\citep{adamw_IlyaLoshchilov2018FixingWD} modifies Adam with more effective weight decay regularization, while Adafactor~\citep{adafactor_shazeer2018adafactor} provides a memory-efficient adaptive method.
Défossez et al. provides a unified formulation for adaptive methods like AdaGrad, Adam, and AdaDelta.

The field continues to evolve with recent innovations including MUON~\citep{muon_jordan2024muon}, LION~\citep{lion_chen2024symbolic},  Sophia~\citep{sophia_liu2023sophia}, and Mars~\citep{mars_yuan2024mars}. These methods represent the cutting edge of adaptive optimization techniques, further advancing efficiency and performance in training deep learning models.

\citet{survey_of_optimization_wang2025survey} give a detailed analysis and survey of optimization methods, we would like to recommend the audience to refer to their paper for a full reference.

\textbf{Remark.} This paper is orthogonal to these works discussed in this section. Our primary contribution is to demonstrate that vanilla mSGD can achieve strong performance on a Transformer architecture when it does not have a heavy-tail problem in gradients. Notably, the optimizers discussed here can also be effectively applied to our proposed DNT network.

\end{document}

%% file: math_commands.tex
\usepackage{amsmath,amsfonts,bm}

\def\eqref#1{equation~\ref{#1}}

\def\1{\bm{1}}

\DeclareMathAlphabet{\mathsfit}{\encodingdefault}{\sfdefault}{m}{sl}
\SetMathAlphabet{\mathsfit}{bold}{\encodingdefault}{\sfdefault}{bx}{n}













%% file: DNT_ICLR_2026.bbl
\begin{thebibliography}{57}
\providecommand{\natexlab}[1]{#1}
\providecommand{\url}[1]{\texttt{#1}}
\expandafter\ifx\csname urlstyle\endcsname\relax
  \providecommand{\doi}[1]{doi: #1}\else
  \providecommand{\doi}{doi: \begingroup \urlstyle{rm}\Url}\fi

\bibitem[Ba et~al.(2016)Ba, Kiros, and Hinton]{layernorm_ba2016layer}
Jimmy~Lei Ba, Jamie~Ryan Kiros, and Geoffrey~E Hinton.
\newblock Layer normalization.
\newblock \emph{arXiv preprint arXiv:1607.06450}, 2016.

\bibitem[Bernstein et~al.(2018)Bernstein, Wang, Azizzadenesheli, and Anandkumar]{signsgd_bernstein2018signsgd}
Jeremy Bernstein, Yu-Xiang Wang, Kamyar Azizzadenesheli, and Animashree Anandkumar.
\newblock signsgd: Compressed optimisation for non-convex problems.
\newblock In \emph{International Conference on Machine Learning}, pp.\  560--569. PMLR, 2018.

\bibitem[Brown et~al.(2020)Brown, Mann, Ryder, Subbiah, Kaplan, Dhariwal, Neelakantan, Shyam, Sastry, Askell, et~al.]{gpt3_brown2020language}
Tom Brown, Benjamin Mann, Nick Ryder, Melanie Subbiah, Jared~D Kaplan, Prafulla Dhariwal, Arvind Neelakantan, Pranav Shyam, Girish Sastry, Amanda Askell, et~al.
\newblock Language models are few-shot learners.
\newblock \emph{Advances in neural information processing systems}, 33:\penalty0 1877--1901, 2020.

\bibitem[Chen et~al.(2024)Chen, Liang, Huang, Real, Wang, Pham, Dong, Luong, Hsieh, Lu, et~al.]{lion_chen2024symbolic}
Xiangning Chen, Chen Liang, Da~Huang, Esteban Real, Kaiyuan Wang, Hieu Pham, Xuanyi Dong, Thang Luong, Cho-Jui Hsieh, Yifeng Lu, et~al.
\newblock Symbolic discovery of optimization algorithms.
\newblock \emph{Advances in neural information processing systems}, 36, 2024.

\bibitem[Chowdhery et~al.(2023)Chowdhery, Narang, Devlin, Bosma, Mishra, Roberts, Barham, Chung, Sutton, Gehrmann, et~al.]{palm_chowdhery2023palm}
Aakanksha Chowdhery, Sharan Narang, Jacob Devlin, Maarten Bosma, Gaurav Mishra, Adam Roberts, Paul Barham, Hyung~Won Chung, Charles Sutton, Sebastian Gehrmann, et~al.
\newblock Palm: Scaling language modeling with pathways.
\newblock \emph{Journal of Machine Learning Research}, 24\penalty0 (240):\penalty0 1--113, 2023.

\bibitem[Dehghani et~al.(2023)Dehghani, Djolonga, Mustafa, Padlewski, Heek, Gilmer, Steiner, Caron, Geirhos, Alabdulmohsin, et~al.]{scaling22b_dehghani2023scaling}
Mostafa Dehghani, Josip Djolonga, Basil Mustafa, Piotr Padlewski, Jonathan Heek, Justin Gilmer, Andreas~Peter Steiner, Mathilde Caron, Robert Geirhos, Ibrahim Alabdulmohsin, et~al.
\newblock Scaling vision transformers to 22 billion parameters.
\newblock In \emph{International Conference on Machine Learning}, pp.\  7480--7512. PMLR, 2023.

\bibitem[Dosovitskiy et~al.(2020)Dosovitskiy, Beyer, Kolesnikov, Weissenborn, Zhai, Unterthiner, Dehghani, Minderer, Heigold, Gelly, et~al.]{vit_dosovitskiy2020image}
Alexey Dosovitskiy, Lucas Beyer, Alexander Kolesnikov, Dirk Weissenborn, Xiaohua Zhai, Thomas Unterthiner, Mostafa Dehghani, Matthias Minderer, Georg Heigold, Sylvain Gelly, et~al.
\newblock An image is worth 16x16 words: Transformers for image recognition at scale.
\newblock In \emph{International Conference on Learning Representations}, 2020.

\bibitem[Dubey et~al.(2024)Dubey, Jauhri, Pandey, Kadian, Al-Dahle, Letman, Mathur, Schelten, Yang, Fan, et~al.]{llama3_dubey2024llama}
Abhimanyu Dubey, Abhinav Jauhri, Abhinav Pandey, Abhishek Kadian, Ahmad Al-Dahle, Aiesha Letman, Akhil Mathur, Alan Schelten, Amy Yang, Angela Fan, et~al.
\newblock The llama 3 herd of models.
\newblock \emph{arXiv preprint arXiv:2407.21783}, 2024.

\bibitem[Duchi et~al.(2011)Duchi, Hazan, and Singer]{adagrad_duchi2011adaptive}
John Duchi, Elad Hazan, and Yoram Singer.
\newblock Adaptive subgradient methods for online learning and stochastic optimization.
\newblock \emph{Journal of machine learning research}, 12\penalty0 (7), 2011.

\bibitem[He et~al.(2016)He, Zhang, Ren, and Sun]{resnet_he2016deep}
Kaiming He, Xiangyu Zhang, Shaoqing Ren, and Jian Sun.
\newblock Deep residual learning for image recognition.
\newblock In \emph{Proceedings of the IEEE conference on computer vision and pattern recognition}, pp.\  770--778, 2016.

\bibitem[He et~al.(2022)He, Chen, Xie, Li, Doll{\'a}r, and Girshick]{mae_he2022masked}
Kaiming He, Xinlei Chen, Saining Xie, Yanghao Li, Piotr Doll{\'a}r, and Ross Girshick.
\newblock Masked autoencoders are scalable vision learners.
\newblock In \emph{Proceedings of the IEEE/CVF conference on computer vision and pattern recognition}, pp.\  16000--16009, 2022.

\bibitem[Henry et~al.(2020)Henry, Dachapally, Pawar, and Chen]{qk_norm_henry2020query}
Alex Henry, Prudhvi~Raj Dachapally, Shubham~Shantaram Pawar, and Yuxuan Chen.
\newblock Query-key normalization for transformers.
\newblock In \emph{Findings of the Association for Computational Linguistics: EMNLP 2020}, pp.\  4246--4253, 2020.

\bibitem[Hinton(2012)]{rmsprop_hinton2012rmsprop}
G~Hinton.
\newblock Rmsprop: divide the gradient by a running average of its recent magnitude.
\newblock \emph{COURSERA: Neural Networks for Machine Learning}, 13, 2012.

\bibitem[Horn \& Johnson(2012)Horn and Johnson]{matrix_analysis_horn2012matrix}
Roger~A Horn and Charles~R Johnson.
\newblock \emph{Matrix analysis}.
\newblock Cambridge university press, 2012.

\bibitem[Ioffe \& Szegedy(2015)Ioffe and Szegedy]{batch_normalization_ioffe2015batch}
Sergey Ioffe and Christian Szegedy.
\newblock Batch normalization: Accelerating deep network training by reducing internal covariate shift.
\newblock In \emph{International conference on machine learning}, pp.\  448--456. PMLR, 2015.

\bibitem[Johnson \& Zhang(2013)Johnson and Zhang]{svrg_johnson2013accelerating}
Rie Johnson and Tong Zhang.
\newblock Accelerating stochastic gradient descent using predictive variance reduction.
\newblock \emph{Advances in neural information processing systems}, 26, 2013.

\bibitem[Jordan et~al.(2024)Jordan, Jin, Boza, Jiacheng, Cesista, Newhouse, and Bernstein]{muon_jordan2024muon}
Keller Jordan, Yuchen Jin, Vlado Boza, You Jiacheng, Franz Cesista, Laker Newhouse, and Jeremy Bernstein.
\newblock Muon: An optimizer for hidden layers in neural networks, 2024.
\newblock URL \url{https://kellerjordan.github.io/posts/muon/}.

\bibitem[Karpathy(2022)]{nanogpt_Karpathy2022}
Andrej Karpathy.
\newblock \text{NanoGPT}.
\newblock \url{https://github.com/karpathy/nanoGPT}, 2022.

\bibitem[Kingma \& Ba(2014)Kingma and Ba]{adam_kingma2014adam}
Diederik~P Kingma and Jimmy Ba.
\newblock Adam: A method for stochastic optimization.
\newblock \emph{arXiv preprint arXiv:1412.6980}, 2014.

\bibitem[LeCun et~al.(1998)LeCun, eon Bottou, Bengio, et~al.]{cnn_lecun1998gradient}
Yann LeCun, L~eon Bottou, Yoshua Bengio, et~al.
\newblock Gradient-based learning applied to document recognition.
\newblock \emph{PROCEEDINGS OF THE IEEE}, pp.\ ~1, 1998.

\bibitem[Li et~al.(2022)Li, Li, Xiong, and Hoi]{blip_li2022blip}
Junnan Li, Dongxu Li, Caiming Xiong, and Steven Hoi.
\newblock Blip: Bootstrapping language-image pre-training for unified vision-language understanding and generation.
\newblock In \emph{International conference on machine learning}, pp.\  12888--12900. PMLR, 2022.

\bibitem[Liu et~al.(2024)Liu, Feng, Xue, Wang, Wu, Lu, Zhao, Deng, Zhang, Ruan, et~al.]{deepseek_v3_liu2024deepseek}
Aixin Liu, Bei Feng, Bing Xue, Bingxuan Wang, Bochao Wu, Chengda Lu, Chenggang Zhao, Chengqi Deng, Chenyu Zhang, Chong Ruan, et~al.
\newblock Deepseek-v3 technical report.
\newblock \emph{arXiv preprint arXiv:2412.19437}, 2024.

\bibitem[Liu et~al.(2023{\natexlab{a}})Liu, Li, Wu, and Lee]{llava_liu2023visual}
Haotian Liu, Chunyuan Li, Qingyang Wu, and Yong~Jae Lee.
\newblock Visual instruction tuning.
\newblock \emph{Advances in neural information processing systems}, 36:\penalty0 34892--34916, 2023{\natexlab{a}}.

\bibitem[Liu et~al.(2023{\natexlab{b}})Liu, Li, Hall, Liang, and Ma]{sophia_liu2023sophia}
Hong Liu, Zhiyuan Li, David Hall, Percy Liang, and Tengyu Ma.
\newblock Sophia: A scalable stochastic second-order optimizer for language model pre-training.
\newblock \emph{arXiv preprint arXiv:2305.14342}, 2023{\natexlab{b}}.

\bibitem[Liu et~al.(2022)Liu, Hu, Lin, Yao, Xie, Wei, Ning, Cao, Zhang, Dong, et~al.]{swinv2_liu2022swin}
Ze~Liu, Han Hu, Yutong Lin, Zhuliang Yao, Zhenda Xie, Yixuan Wei, Jia Ning, Yue Cao, Zheng Zhang, Li~Dong, et~al.
\newblock Swin transformer v2: Scaling up capacity and resolution.
\newblock In \emph{Proceedings of the IEEE/CVF Conference on Computer Vision and Pattern Recognition}, pp.\  12009--12019, 2022.

\bibitem[Loshchilov \& Hutter(2019)Loshchilov and Hutter]{adamw_IlyaLoshchilov2018FixingWD}
Ilya Loshchilov and Frank Hutter.
\newblock Fixing weight decay regularization in adam.
\newblock In \emph{International Conference on Learning Representations}, 2019.

\bibitem[Loshchilov et~al.(2025)Loshchilov, Hsieh, Sun, and Ginsburg]{ngpt_loshchilov2024ngpt}
Ilya Loshchilov, Cheng-Ping Hsieh, Simeng Sun, and Boris Ginsburg.
\newblock ngpt: Normalized transformer with representation learning on the hypersphere.
\newblock \emph{The Thirteenth International Conference on Learning Representations}, 2025.

\bibitem[Nesterov(1983)]{nesterov1983method}
Yurii Nesterov.
\newblock A method for unconstrained convex minimization problem with the rate of convergence o (1/k\^{} 2).
\newblock In \emph{Doklady an ussr}, volume 269, pp.\  543--547, 1983.

\bibitem[Nesterov(2013)]{nestorov_nesterov2013introductory}
Yurii Nesterov.
\newblock \emph{Introductory lectures on convex optimization: A basic course}, volume~87.
\newblock Springer Science \& Business Media, 2013.

\bibitem[Nesterov(1998)]{nesterov1998introductory}
Yurri Nesterov.
\newblock Introductory lectures on convex programming volume i: Basic course.
\newblock 1998.

\bibitem[Noci et~al.(2022)Noci, Anagnostidis, Biggio, Orvieto, Singh, and Lucchi]{rank_collapse_noci2022signal}
Lorenzo Noci, Sotiris Anagnostidis, Luca Biggio, Antonio Orvieto, Sidak~Pal Singh, and Aurelien Lucchi.
\newblock Signal propagation in transformers: Theoretical perspectives and the role of rank collapse.
\newblock \emph{Advances in Neural Information Processing Systems}, 35:\penalty0 27198--27211, 2022.

\bibitem[Paszke et~al.(2019)Paszke, Gross, Massa, Lerer, Bradbury, Chanan, Killeen, Lin, Gimelshein, Antiga, et~al.]{pytorch_paszke2019pytorch}
Adam Paszke, Sam Gross, Francisco Massa, Adam Lerer, James Bradbury, Gregory Chanan, Trevor Killeen, Zeming Lin, Natalia Gimelshein, Luca Antiga, et~al.
\newblock Pytorch: An imperative style, high-performance deep learning library.
\newblock \emph{Advances in neural information processing systems}, 32, 2019.

\bibitem[Peebles \& Xie(2023)Peebles and Xie]{dit_peebles2023scalable}
William Peebles and Saining Xie.
\newblock Scalable diffusion models with transformers.
\newblock In \emph{Proceedings of the IEEE/CVF International Conference on Computer Vision}, pp.\  4195--4205, 2023.

\bibitem[Qi et~al.(2023)Qi, Wang, Chen, Shi, and Zhang]{lipsformer_qilipsformer}
Xianbiao Qi, Jianan Wang, Yihao Chen, Yukai Shi, and Lei Zhang.
\newblock Lipsformer: Introducing lipschitz continuity to vision transformers.
\newblock In \emph{The Eleventh International Conference on Learning Representations}, 2023.

\bibitem[Qi et~al.(2025{\natexlab{a}})Qi, He, Ye, Li, Zi, Dai, Zou, and Xiao]{qi_taming_transformer_qitaming}
Xianbiao Qi, Yelin He, Jiaquan Ye, Chun-Guang Li, Bojia Zi, Xili Dai, Qin Zou, and Rong Xiao.
\newblock Taming transformer without using learning rate warmup.
\newblock \emph{The Thirteenth International Conference on Learning Representations}, 2025{\natexlab{a}}.

\bibitem[Qi et~al.(2025{\natexlab{b}})Qi, Ye, He, Li, Zi, Dai, Zou, and Xiao]{stable_transformer_qi2025stabletransformer}
Xianbiao Qi, Jiaquan Ye, Yelin He, Chun-Guang Li, Bojia Zi, Xili Dai, Qin Zou, and Rong Xiao.
\newblock Stable-transformer: Towards a stable transformer training, 2025{\natexlab{b}}.
\newblock URL \url{https://openreview.net/forum?id=lkRjnNW0gb}.

\bibitem[Radford et~al.(2018)Radford, Narasimhan, Salimans, Sutskever, et~al.]{gpt1_radford2018improving}
Alec Radford, Karthik Narasimhan, Tim Salimans, Ilya Sutskever, et~al.
\newblock Improving language understanding by generative pre-training.
\newblock 2018.

\bibitem[Radford et~al.(2019)Radford, Wu, Child, Luan, Amodei, Sutskever, et~al.]{gpt2_radford2019language}
Alec Radford, Jeffrey Wu, Rewon Child, David Luan, Dario Amodei, Ilya Sutskever, et~al.
\newblock Language models are unsupervised multitask learners.
\newblock \emph{OpenAI blog}, 1\penalty0 (8):\penalty0 9, 2019.

\bibitem[Ramesh et~al.(2021)Ramesh, Pavlov, Goh, Gray, Voss, Radford, Chen, and Sutskever]{dalle1_ramesh2021zero}
Aditya Ramesh, Mikhail Pavlov, Gabriel Goh, Scott Gray, Chelsea Voss, Alec Radford, Mark Chen, and Ilya Sutskever.
\newblock Zero-shot text-to-image generation.
\newblock In \emph{International conference on machine learning}, pp.\  8821--8831. Pmlr, 2021.

\bibitem[Robbins \& Monro(1951)Robbins and Monro]{sgd_robbins1951stochastic}
Herbert Robbins and Sutton Monro.
\newblock A stochastic approximation method.
\newblock \emph{The annals of mathematical statistics}, pp.\  400--407, 1951.

\bibitem[Shazeer \& Stern(2018)Shazeer and Stern]{adafactor_shazeer2018adafactor}
Noam Shazeer and Mitchell Stern.
\newblock Adafactor: Adaptive learning rates with sublinear memory cost.
\newblock In \emph{International Conference on Machine Learning}, pp.\  4596--4604. PMLR, 2018.

\bibitem[Simsekli et~al.(2019)Simsekli, Sagun, and Gurbuzbalaban]{heavy_tail_simsekli2019tail}
Umut Simsekli, Levent Sagun, and Mert Gurbuzbalaban.
\newblock A tail-index analysis of stochastic gradient noise in deep neural networks.
\newblock In \emph{International Conference on Machine Learning}, pp.\  5827--5837. PMLR, 2019.

\bibitem[Sutskever et~al.(2013)Sutskever, Martens, Dahl, and Hinton]{sutskever2013importance}
Ilya Sutskever, James Martens, George Dahl, and Geoffrey Hinton.
\newblock On the importance of initialization and momentum in deep learning.
\newblock In \emph{International conference on machine learning}, pp.\  1139--1147. PMLR, 2013.

\bibitem[Tao(2012)]{random_matrix_theory_tao2012topics}
Terence Tao.
\newblock \emph{Topics in random matrix theory}, volume 132.
\newblock American Mathematical Soc., 2012.

\bibitem[Team(2023)]{qwen_team2023qwen}
Qwen Team.
\newblock Qwen technical report.
\newblock \emph{arXiv preprint arXiv:2309.16609}, 2023.

\bibitem[Touvron et~al.(2023)Touvron, Martin, Stone, Albert, Almahairi, Babaei, Bashlykov, Batra, Bhargava, Bhosale, et~al.]{llama2_touvron2023llama}
Hugo Touvron, Louis Martin, Kevin Stone, Peter Albert, Amjad Almahairi, Yasmine Babaei, Nikolay Bashlykov, Soumya Batra, Prajjwal Bhargava, Shruti Bhosale, et~al.
\newblock Llama 2: Open foundation and fine-tuned chat models.
\newblock \emph{arXiv preprint arXiv:2307.09288}, 2023.

\bibitem[Vaswani et~al.(2017)Vaswani, Shazeer, Parmar, Uszkoreit, Jones, Gomez, Kaiser, and Polosukhin]{transformer_vaswani2017attention}
Ashish Vaswani, Noam Shazeer, Niki Parmar, Jakob Uszkoreit, Llion Jones, Aidan~N Gomez, {\L}ukasz Kaiser, and Illia Polosukhin.
\newblock Attention is all you need.
\newblock \emph{Advances in neural information processing systems}, 30, 2017.

\bibitem[Vershynin(2018)]{vershynin2018high}
Roman Vershynin.
\newblock \emph{High-dimensional probability: An introduction with applications in data science}, volume~47.
\newblock Cambridge university press, 2018.

\bibitem[Wang \& Choromanska(2025)Wang and Choromanska]{survey_of_optimization_wang2025survey}
Jing Wang and Anna Choromanska.
\newblock A survey of optimization methods for training dl models: Theoretical perspective on convergence and generalization.
\newblock \emph{arXiv preprint arXiv:2501.14458}, 2025.

\bibitem[Wightman(2019)]{timm_rw2019timm}
Ross Wightman.
\newblock Pytorch image models.
\newblock \url{https://github.com/rwightman/pytorch-image-models}, 2019.

\bibitem[Xie et~al.(2024)Xie, Zhou, Li, Lin, and Yan]{adan_xie2024adan}
Xingyu Xie, Pan Zhou, Huan Li, Zhouchen Lin, and Shuicheng Yan.
\newblock Adan: Adaptive nesterov momentum algorithm for faster optimizing deep models.
\newblock \emph{IEEE Transactions on Pattern Analysis and Machine Intelligence}, 2024.

\bibitem[You et~al.(2017)You, Gitman, and Ginsburg]{lars_you2017large}
Yang You, Igor Gitman, and Boris Ginsburg.
\newblock Large batch training of convolutional networks.
\newblock \emph{arXiv preprint arXiv:1708.03888}, 2017.

\bibitem[Yuan et~al.(2024)Yuan, Liu, Wu, Zhou, and Gu]{mars_yuan2024mars}
Huizhuo Yuan, Yifeng Liu, Shuang Wu, Xun Zhou, and Quanquan Gu.
\newblock Mars: Unleashing the power of variance reduction for training large models.
\newblock \emph{arXiv preprint arXiv:2411.10438}, 2024.

\bibitem[Zhai et~al.(2023)Zhai, Likhomanenko, Littwin, Busbridge, Ramapuram, Zhang, Gu, and Susskind]{stabilizing_transformer_zhai2023stabilizing}
Shuangfei Zhai, Tatiana Likhomanenko, Etai Littwin, Dan Busbridge, Jason Ramapuram, Yizhe Zhang, Jiatao Gu, and Joshua~M Susskind.
\newblock Stabilizing transformer training by preventing attention entropy collapse.
\newblock In \emph{International Conference on Machine Learning}, pp.\  40770--40803. PMLR, 2023.

\bibitem[Zhang \& Sennrich(2019)Zhang and Sennrich]{rmsnorm_zhang2019root}
Biao Zhang and Rico Sennrich.
\newblock Root mean square layer normalization.
\newblock \emph{Advances in Neural Information Processing Systems}, 32, 2019.

\bibitem[Zhang et~al.(2020)Zhang, Karimireddy, Veit, Kim, Reddi, Kumar, and Sra]{zhangjingzhao_zhang2020adaptive}
Jingzhao Zhang, Sai~Praneeth Karimireddy, Andreas Veit, Seungyeon Kim, Sashank Reddi, Sanjiv Kumar, and Suvrit Sra.
\newblock Why are adaptive methods good for attention models?
\newblock \emph{Advances in Neural Information Processing Systems}, 33:\penalty0 15383--15393, 2020.

\bibitem[Zhu et~al.(2025)Zhu, Chen, He, LeCun, and Liu]{tears_zhu2025transformers}
Jiachen Zhu, Xinlei Chen, Kaiming He, Yann LeCun, and Zhuang Liu.
\newblock Transformers without normalization.
\newblock \emph{arXiv preprint arXiv:2503.10622}, 2025.

\end{thebibliography}
